\documentclass[twoside,11pt]{article}

\usepackage[preprint]{jmlr2e}

\jmlrheading{}{2026}{1-\pageref{LastPage}}{8/26}{--/--}{}{Elad Aigner-Horev and Daniel Rosenberg and Roi Weiss}
\ShortHeadings{Sample Complexity of Robust PAC Learning}{Aigner-Horev, Rosenberg, and Weiss}
\usepackage{lastpage}
\firstpageno{1}

\usepackage{amsfonts}
\usepackage{amsmath}
\usepackage{amssymb}
 \usepackage{mathtools}
\usepackage{bbm}
\usepackage{color}

\newcommand{\EE}{\mathbb{E}}
\renewcommand{\Pr}{\mathbb{P}}
\newcommand{\R}{\mathbb{R}}
\newcommand{\X}{\mathcal{X}}
\newcommand{\Y}{\mathcal{Y}}
\newcommand{\hypoclass}{\mathcal{H}}
\newcommand{\us}{\mathcal{U}}
\newcommand{\usa}[2]{\mathcal{U}_{#2}(#1)}
\newcommand{\risk}[3]{\mathcal{R}_{#1}(#2; #3)}

\newcommand{\ind}[1]{\mathbbm{1}[#1]}

\newcommand{\Pd}{P}
\newcommand{\Ped}{P_n}

\newcommand{\sampcomp}{\mathcal M}
\newcommand{\err}[2]{\textup{err}_{#2}(#1)}

\newcommand{\DKL}{D_{\textup{KL}}}
\newcommand{\VC}{\textup{VC}}
\newcommand{\eps}{\varepsilon}

\newcommand{\hamming}{d_{\textup{H}}}

\newcommand{\ch}{\mathbbm{1}}

\newcommand{\q}[2]{\mathcal{Q}_{#1}(#2)}

\newcommand{\e}{\textup{e}}
\newcommand{\agn}{\textup{agn}}
\newcommand{\real}{\textup{real}}

\newcommand{\cep}{\bar{\hat p}}
\newcommand{\cp}{\bar{p}}
\newcommand{\cd}{\bar{\Delta}}
\newcommand{\erm}{\hat h}
\newcommand{\GammaMac}{\Gamma_n}

\DeclareMathOperator*{\argmin}{arg\,min}

\begin{document}

\title{The Sample Complexity of Distributionally Robust PAC Learning under Cressie--Read Divergences}

\author{\name Elad Aigner-Horev 
\email horev@ariel.ac.il \\
\addr School of Computer Science and AI\\
Ariel University\\
40800 Ariel, Israel 
\AND
\name Daniel Rosenberg 
\email danielr@ariel.ac.il \\
\addr School of Computer Science and AI\\
Ariel University\\
40800 Ariel, Israel 
\AND
\name Roi Weiss 
\email roiw@ariel.ac.il \\
\addr School of Computer Science and AI\\
Ariel University\\
40800 Ariel, Israel 
}

\editor{}

\maketitle

\begin{abstract}%
We study distributionally robust PAC learning for the \(0\)--\(1\)-loss, where adversarial perturbations of the data distribution are constrained by a Cressie--Read divergence of order \(k>1\) and radius \(\rho\geq0\).
For hypothesis classes with VC dimension \(d\), we establish
realizable and agnostic sample-complexity bounds tight up to
constant and logarithmic factors, respectively; ordinary empirical
risk minimization attains both rates up to logarithmic factors.
For target accuracy \(\eps\in(0,1)\) and confidence
\(\delta\in(0,1)\), their respective orders are
\[
\max\!\left\{\frac{1}{\eps},
\frac{\rho^{\frac 1{k-1}}}{\eps^{k_\star}}
\right\}\cdot(d+\log \delta^{-1})
\qquad\text{and}\qquad
\max\!\left\{\frac{1}{\eps^2},
\frac{\rho^{\frac1{k-1}}}{\eps^{k_\star\vee 2}}
\right\}\cdot(d+\log \delta^{-1}),
\]
where \(k_\star={k}/{(k-1)}\).
For every fixed \(\rho>0\), robustness changes the realizable
\(\eps\)-dependence from \(\eps^{-1}\) to
\(\eps^{-k_\star}\) as \(\eps\downarrow0\).
In the agnostic case, for \(1<k<2\), robustness changes the
\(\eps\)-dependence from \(\eps^{-2}\) to
\(\eps^{-k_\star}\), whereas for \(k\geq2\) the exponent remains the
classical \(2\), with nontrivial \(\rho\)-dependence.

Building on the known scalar reduction of robust \(0\)--\(1\) risk
to ordinary classification error, our analysis reveals a
scale-sensitive interaction between the statistical estimation of
classification error and its amplification by robustness, sharply
explaining the transition in the agnostic rate.
We extend the previously studied \(\chi^2\)-divergence case to every
Cressie--Read order \(k>1\), close its upper--lower gaps, and recover
standard PAC learning rates as \(\rho\to0\), unlike previous bounds
that fail to interpolate correctly in this limit.
\bigskip
\end{abstract}


\newpage
\section{Introduction}
\label{sec:intro}

The degradation of predictive performance when training and test
distributions differ is a classical concern in statistical learning
\citep{S00,DBCP06,bickel2007discriminative,hu2018does}.
Such distributional shifts may arise for benign and expected reasons,
such as natural changes in the underlying population, or from more
adversarial interventions affecting the test data or application
domain. The desire to train machine-learning models in a
{\sl robust} fashion at the outset, with the aim of anticipating
unfamiliar testing scenarios, is therefore clear.

Within this broader program, the study of
{\sl distributionally robust PAC learnability}, defined below,
provides one formal approach to mitigating performance degradation
under distributional shift. 
The PAC guarantee is uniform over all data-generating distributions: for
every target accuracy and confidence level, a single sample-size
threshold must suffice for all data-generating distributions to
ensure, with the prescribed confidence, that the learned classifier's
robust risk is within the target accuracy of the infimum robust risk
over the hypothesis class.

For the \(0\)--\(1\)-loss under \(f\)-divergence uncertainty, this
formulation has a special structure. As established in prior work and
reviewed below, the robust risk of a classifier is a nondecreasing
scalar function of its ordinary classification error. Consequently,
any classifier minimizing empirical classification error also
minimizes the corresponding empirical robust risk. This does not,
however, make the statistical problem equivalent to classical PAC
learning. To control robust risk to a prescribed accuracy, the
ordinary classification errors must be estimated accurately enough
that their images under the scalar map differ by no more than that
accuracy. The sensitivity of the map therefore determines how
accurately ordinary errors must be estimated, and hence how many
samples are required. Because this sensitivity varies across
ordinary-error scales, the resulting sample complexity can differ
from its classical counterpart.

We study this problem for the Cressie--Read family
\citep{CR84}, a one-parameter family of \(f\)-divergences indexed by
\(k>1\) that contains the normalized \(\chi^2\)-divergence as the
case \(k=2\). As our analysis shows, the order \(k\) governs the
scale sensitivity of the robust-risk mapping and thereby
which ordinary-error scales determine the sample complexity. To the
best of our knowledge, \citet{zhou2023sample} initiated the
VC-theoretic study of distributionally robust PAC learnability for
the \(0\)--\(1\)-loss in the \(\chi^2\)-divergence setting. 
In the present work, we
characterize, up to logarithmic factors, the corresponding realizable
and agnostic sample complexities for every fixed Cressie--Read order
\(k>1\).

\noindent\emph{Notation.}
For \(a,b\geq0\), write \(a\lesssim_k b\) if
\(a\leq C_k b\), with \(\gtrsim_k\) and \(\asymp_k\)
defined analogously; subscripts record the permitted parameter
dependence of $C_k$ and are omitted for universal or contextually fixed
constants. The same convention applies to
\(\tilde O,\tilde\Omega,\tilde\Theta\), which also
suppress polylogarithmic factors in the displayed parameters.
We write \(a\vee b=\max\{a,b\}\), \(a\wedge b=\min\{a,b\}\),
and \(a\ll b\) when \(a/b\to0\) in the relevant regime.
$A^c$ denotes the complement of a subset $A$. 

\subsection{Distributionally Robust PAC Learnability: Problem Setup}

Let \(\X\) be a measurable instance space, let
\(\Y=\{-1,1\}\), and let \(\hypoclass\) be a class of measurable
classifiers \(h:\X\to\Y\).%
\footnote{Whenever a sample-dependent event \(E\) is not known to
be measurable, a statement that \(E\) holds with probability at
least \(1-\delta\) is understood to mean
\(\Pr^*(E^c)\leq\delta\), where
\[
\Pr^*(A)
\coloneqq
\inf\bigl\{\Pr(B):A\subseteq B,\ B\text{ is measurable}\bigr\}
\]
denotes outer probability under the relevant sampling law. This
convention applies, in particular, to events involving suprema over
\(\hypoclass\) and to pointwise selections of ERM or DRERM
minimizers.}
Denote by
\[
\mathcal G
\coloneqq
\{g:\X\to\Y:g\text{ is measurable}\}
\]
the class of all measurable binary classifiers on \(\X\).
A learning rule based on \(n\) labeled observations is a map
\[
r_n:(\X\times\Y)^n\to\mathcal G.
\]
The learning rule is called \emph{proper} with respect to
\(\hypoclass\) if
\[
r_n(S_n)\in\hypoclass
\qquad
\text{for every }S_n\in(\X\times\Y)^n;
\]
a learning rule that is not proper is called \emph{improper}.
Unless explicitly qualified as proper, the sample-complexity
quantities considered below allow learning rules whose outputs need
not belong to \(\hypoclass\).

We work throughout with the \(0\)--\(1\)-loss. For a distribution
\(P\) on \(\X\times\Y\), the classification error
of \(g\in\mathcal G\) is
\[
\err{g}{P}
\coloneqq
\Pr_{(x,y)\sim P}[g(x)\neq y].
\]
Given an uncertainty set \(\us(P)\) of distributions on
\(\X\times\Y\), define the \emph{distributionally robust risk}
\[
\risk{\us}{g}{P}
\coloneqq
\sup_{Q\in\us(P)}\err{g}{Q}.
\]
Since \(P\) is unknown, a learner observes an i.i.d.\ sample
\(S_n\sim P^n\) and selects a classifier
\(r_n(S_n)\in\mathcal G\), with the aim of controlling the
\emph{robust excess risk}
\[
\risk{\us}{r_n(S_n)}{P}
-
\inf_{h\in\hypoclass}\risk{\us}{h}{P}.
\]

We consider uncertainty sets given by \(f\)-divergence
balls~\citep{ali1966general,csiszar1967information,ben2013robust}:
\[
\usa{P}{f,\rho}
\coloneqq
\left\{
Q\ll P:
D_f(Q\|P)\leq\rho
\right\},
\qquad
\rho\geq0,
\]
where \(Q\ll P\) denotes absolute continuity and
\[
D_f(Q\|P)
\coloneqq
\EE_P\left[
f\left(\frac{dQ}{dP}\right)
\right]
\]
for a convex function \(f:[0,\infty)\to\R\) satisfying
\(f(1)=0\). 
Thus, these uncertainty sets model likelihood-ratio reweightings of
the joint data distribution. In particular,
\[
P(A)=0
\quad\Longrightarrow\quad
Q(A)=0 \,\,\, \textup{for every}\,\,\, Q\in\usa{P}{f,\rho}.
\]

Our main results concern the Cressie--Read
family~\citep{CR84}, 
\begin{equation}
\label{eq:CR_intro}
f_k(t)\coloneqq\frac{t^k-kt+k-1}{k(k-1)},
\qquad k\in(1,\infty).
\end{equation}
The Cressie--Read family of divergences contains, as special cases, the {\sl normalized $\chi^2$-divergence}
\(
f_2(t)={(t-1)^2}/{2},
\)
and, at the limit $k\to1$, the {\em Kullback--Leibler divergence} (KL divergence)
\(
f_1(t)=t\log t-t+1.
\)
Specializing the above notation to the Cressie--Read family, we write
\[
\us_{k,\rho}(P)\coloneqq\us_{f_k,\rho}(P)
\qquad\text{and}\qquad
\risk{k,\rho}{g}{P}
\coloneqq
\risk{\us_{f_k,\rho}}{g}{P},
\qquad g\in\mathcal G.
\]

\begin{definition}[Distributionally robust PAC learning and sample complexities]
\label{def:dist-rob-pac}
A distribution \(P\) on \(\X\times\Y\) is called
\emph{realizable with respect to \(\hypoclass\)} if there exists
\(h^\star\in\hypoclass\) such that
\[
\err{h^\star}{P}=0.
\]

Fix \(k>1\), \(\rho\geq0\), and \(\eps,\delta\in(0,1)\).

\begin{enumerate}
\item
The agnostic distributionally robust PAC sample complexity
\[
\sampcomp_{k,\rho}^{\agn}(\eps,\delta;\hypoclass)
\]
is the infimum of all integers \(n\geq1\) for which there exists a
learning rule
\[
r_n:(\X\times\Y)^n\to\mathcal G
\]
such that, for every distribution \(P\) on \(\X\times\Y\),
\[
\Pr_{S_n\sim P^n}
\left[
\risk{k,\rho}{r_n(S_n)}{P}
-
\inf_{h\in\hypoclass}\risk{k,\rho}{h}{P}
\leq\eps
\right]
\geq1-\delta.
\]

\item
The realizable distributionally robust PAC sample complexity
\[
\sampcomp_{k,\rho}^{\real}(\eps,\delta;\hypoclass)
\]
is the infimum of all integers \(n\geq1\) for which there exists a
learning rule
\[
r_n:(\X\times\Y)^n\to\mathcal G
\]
such that, for every distribution \(P\) that is realizable with
respect to \(\hypoclass\),
\[
\Pr_{S_n\sim P^n}
\left[
\risk{k,\rho}{r_n(S_n)}{P}
\leq\eps
\right]
\geq1-\delta.
\]
\end{enumerate}

The infimum of the empty set is understood to be \(+\infty\). For
fixed \(k>1\) and \(\rho\geq0\), the class \(\hypoclass\) is called
agnostically, respectively realizably, distributionally robust
PAC learnable if the corresponding sample complexity is finite for
every \(\eps,\delta\in(0,1)\).
\end{definition}

For the classical PAC learning model, write
\(\sampcomp_{\mathrm{pac}}^{\agn}(\eps,\delta;\hypoclass)\) and
\(\sampcomp_{\mathrm{pac}}^{\real}(\eps,\delta;\hypoclass)\) for the
corresponding agnostic and realizable sample
complexities. Since \(f_k\) is strictly convex for every \(k>1\),
\[
\us_{k,0}(P)=\{P\},
\qquad
\risk{k,0}{g}{P}=\err{g}{P},
\qquad g\in\mathcal G.
\]
Consequently, Definition~\ref{def:dist-rob-pac} reduces to the
classical PAC learning model when \(\rho=0\).

Writing \(\Ped\) for the empirical distribution of \(S_n\), an
\emph{\(\hypoclass\)-empirical risk minimizer}
(\(\hypoclass\)-ERM) is any classifier
\begin{equation}
\label{eq:ERM}
\erm
\in
\argmin_{h\in\hypoclass}\err{h}{\Ped}.
\end{equation}
An \(\hypoclass\)-ERM is proper by definition.

Suppose that \(\VC(\hypoclass)=d<\infty\). For every such class, the
upper bounds corresponding to the right-hand sides below hold. If,
in addition, \(|\hypoclass|\geq3\), matching lower bounds give
\citep[see, e.g.,][]{boucheron2005theory,hanneke2016optimal}
\begin{equation}
\label{eq:PAC-complexities}
\begin{aligned}
\sampcomp_{\mathrm{pac}}^{\agn}
(\eps,\delta;\hypoclass)
&=
\Theta\left(
\frac{d+\log(1/\delta)}{\eps^2}
\right),
\\[4pt]
\sampcomp_{\mathrm{pac}}^{\real}
(\eps,\delta;\hypoclass)
&=
\Theta\left(
\frac{d+\log(1/\delta)}{\eps}
\right).
\end{aligned}
\end{equation}
The condition \(|\hypoclass|\geq3\) is needed only for the general
lower-bound statement; the smaller exceptional classes are described
by \citet{hanneke2016optimal}.

The agnostic upper bound in \eqref{eq:PAC-complexities} is attained
by any \(\hypoclass\)-ERM. In the realizable setting, any
\(\hypoclass\)-ERM satisfies the standard upper bound
\[
O\left(
\frac{
d\log(\e/\eps)+\log(1/\delta)
}{\eps}
\right),
\]
whereas the log-free upper bound in
\eqref{eq:PAC-complexities} is attained by the learner constructed
by \citet{hanneke2016optimal}.

For the Cressie--Read robust risk, an
\emph{\(\hypoclass\)-distributionally robust empirical risk minimizer}
(\(\hypoclass\)-DRERM) is any
\begin{equation}
\label{eq:DRERM}
\erm
\in
\argmin_{h\in\hypoclass}
\risk{k,\rho}{h}{\Ped}.
\end{equation}
An \(\hypoclass\)-DRERM is also proper. Hereafter, ERM and DRERM
refer to \(\hypoclass\)-ERM and \(\hypoclass\)-DRERM, respectively,
unless stated otherwise.

\subsection{Relevant Prior Results}
\label{sec:prior}

For a measurable event \(A\), prior work on distributionally robust
chance constraints showed that its worst-case probability
\[
\sup_{Q\in\usa{P}{f,\rho}}Q(A)
\]
is determined by \(A\) only through \(P(A)\); see
\citet{hu2013ambiguous} and
\citet[Theorem~2]{jiang2016data}. The latter result applies, in
particular, to the Cressie--Read family. 
For the \(0\)--\(1\)-loss, \citet[Theorem~1]{hu2018does}
established the corresponding result for classifier error events:
below the maximal robust-risk value \(1\), strict comparisons of
robust risks are equivalent to strict comparisons of ordinary
classification errors.
The same relationship holds for the empirical risks. Hence,
in the present notation,
\begin{equation}
\label{eq:ERM-DRERM}
\argmin_{h\in\hypoclass}
\err{h}{\Ped}
\subseteq
\argmin_{h\in\hypoclass}
\risk{k,\rho}{h}{\Ped}.
\end{equation}
Thus, every ordinary ERM is also a DRERM. The inclusion can be strict
only through saturation of the robust objective at \(1\);
Section~\ref{sec:EI} gives the precise
saturation characterization.

These structural results settle the empirical optimization question:
ordinary ERM already solves the empirical robust-risk minimization
problem. We next review the finite-sample results most relevant to the
remaining statistical question.

\medskip
\citet{duchi2021learning} study robust-risk
estimation and optimization for bounded losses under Cressie--Read
uncertainty sets with $k>1$. Specializing their Theorem~2 to the
\(0\)--\(1\)-loss gives, for every fixed \(h\in\hypoclass\), with
probability at least \(1-\delta\),
\begin{equation}
\label{eq:DN}
\left|
\risk{k,\rho}{h}{\Pd}
-
\risk{k,\rho}{h}{\Ped}
\right|
=
\tilde O_{\rho,k,\delta}\!\left(
n^{-1/(k_\star\vee2)}
\right),
\qquad
k_\star=\frac{k}{k-1}.
\end{equation}
They also establish matching minimax lower bounds in \(n\) for
robust-risk estimation and optimization. Thus, the robust-risk
estimation rate is \(n^{-1/k_\star}\) for \(1<k<2\) and
\(n^{-1/2}\) for \(k\geq2\).

The explicit bound underlying~\eqref{eq:DN} diverges as
\(\rho\downarrow0\). Moreover, their uniform bounds use
\(L_\infty\)-covering numbers of the induced loss class, which may be
infinite for VC classes. Consequently, this argument does not
yield distributionally robust PAC guarantees for arbitrary
VC classes; see Section~\ref{sec:compare-prior-detailed}.

For uncertainty sets determined by $\chi^2$-divergence balls ($k=2$), 
\citet{zhou2023sample} initiated the study of distributionally robust PAC learning for VC classes and the \(0\)--\(1\)-loss, where they establish the upper bounds  
\begin{equation}
\label{eq:Liu-Zhou-upper-rates}
\sampcomp_{2,\rho}^{\agn}(\eps,\delta;\hypoclass)  = \tilde O_\rho\!\left(
\frac{d}{\eps^{4}} \right) \quad \text{and} \quad 
\sampcomp_{2,\rho}^{\real}(\eps,\delta;\hypoclass)  = \tilde O_\rho\left(\frac{d}{\eps^2}\right).
\end{equation}
They also establish the lower bounds
\begin{equation}
\label{eq:Liu-Zhou-lower-rates}
\sampcomp_{2,\rho}^{\agn}(\eps,\delta;\hypoclass)= \Omega_\rho\left(\frac{d}{\eps^2}\right) \quad \text{and} \quad \sampcomp_{2,\rho}^{\real}(\eps,\delta;\hypoclass) = \Omega_\rho\left(\frac{d}{\eps} \right),
\end{equation}
where \(\VC(\hypoclass)=d\); the agnostic lower bound is proved for
\(\rho\) below a universal constant.

In subsequent work, \citet{zhou2026rademacher} considered bounded
loss functions and the Cressie--Read family with \(k>1\).
For \(\rho>0\), under the bounded-domain and upper-semicontinuity
assumptions of their theorem, applying their result to the induced
\(0\)--\(1\)-loss class yields that, with probability at least
\(1-\delta\), a DRERM \(\erm\) as in~\eqref{eq:DRERM} satisfies
\begin{equation}
\label{eq:LZ}
\risk{k,\rho}{\erm}{\Pd}
-
\inf_{h\in\hypoclass}
\risk{k,\rho}{h}{\Pd}
\leq
\tilde O_{\rho,k,\delta}\left(
R_n(\psi\circ\Phi)^{1/k_\star}
+
n^{-1/(2k_\star)}
\right),
\end{equation}
where \(R_n\) denotes expected Rademacher complexity and
\(\psi\) and \(\Phi\) are as defined in the cited paper. Even without
further bounding the class-dependent first term, the additive term
in~\eqref{eq:LZ} decays as \(n^{-1/(2k_\star)}\), which is slower
than the \(n^{-1/(k_\star\vee2)}\) robust excess-risk rate obtained
in \citet{duchi2021learning} under finite \(L_\infty\)-covering
assumptions. The bounded-domain and upper-semicontinuity assumptions
are not automatic for an arbitrary measurable VC class, so this
result does not directly cover the full generality considered here.
Moreover, as with~\eqref{eq:DN} and the agnostic upper bound in
\eqref{eq:Liu-Zhou-upper-rates}, the dependence on \(\rho\) hidden
in~\eqref{eq:LZ} does not recover the classical rate as
\(\rho\downarrow0\).

\medskip
These finite-sample results leave open how the local sensitivity of
the scalar robust-risk map interacts with the scale-sensitive
estimation of ordinary classification error. What realizable and
agnostic sample complexities does this interaction induce for every
Cressie--Read order \(k>1\), in particular closing the
\(\chi^2\)-divergence gaps between
\eqref{eq:Liu-Zhou-upper-rates} and
\eqref{eq:Liu-Zhou-lower-rates}? Do the resulting rates recover their
classical PAC counterparts as \(\rho\downarrow0\)?

\subsection{Main Results: Abridged Formulations}
We first state an abridged formulation of our main result.

\begin{theorem}[Main learnability result: abridged formulation]
\label{thm:main:abridged}
There exists a universal constant \(\delta_0>0\) such that, for every
\(k>1\), there exists \(\eps_k>0\) for which the following holds.
Let \(\eps\in(0,\eps_k]\), \(\delta\in(0,\delta_0]\), and
\(\rho\geq0\), and let \(\hypoclass\) be a binary hypothesis class
with
\[
\VC(\hypoclass)=d<\infty
\qquad\text{and}\qquad
|\hypoclass|\geq3.
\]
Then
\begin{equation}
\label{eq:PAC-robust-complexities}
\begin{aligned}
\sampcomp_{k,\rho}^{\real}(\eps,\delta;\hypoclass)
&=
\Theta_k\!\left(
\max\!\left\{
\frac{1}{\eps},
\frac{\rho^{1/(k-1)}}{\eps^{k_\star}}
\right\}
\bigl(d+\log(1/\delta)\bigr)
\right),
\\[5pt]
\sampcomp_{k,\rho}^{\agn}(\eps,\delta;\hypoclass)
&=
\tilde\Theta_k\!\left(
\max\!\left\{
\frac{1}{\eps^2},
\frac{\rho^{1/(k-1)}}{\eps^{k_\star\vee2}}
\right\}
\bigl(d+\log(1/\delta)\bigr)
\right).
\end{aligned}
\end{equation}
The upper-bound directions in \eqref{eq:PAC-robust-complexities}
hold for every binary hypothesis class of finite VC dimension,
without the assumption \(|\hypoclass|\geq3\).
\end{theorem}

DRERM attains both displayed rates up to logarithmic factors. The
log-free realizable refinement follows by transferring an optimal
classical realizable PAC learner through the event-inflation map; see
Section~\ref{sec:PAC_rates}.

The qualitative scalar reduction of robust \(0\)--\(1\) risk and
the ERM--DRERM relationship in~\eqref{eq:ERM-DRERM} are known.
Our contribution is quantitative. By analyzing the Cressie--Read
robust-risk map directly, rather than applying empirical-process
control after Shapiro's dual reformulation as in
\citet{duchi2021learning} and
\citet{zhou2023sample,zhou2026rademacher}, we retain its dependence
on the underlying ordinary-error scale. This enables us to determine
how changes in ordinary classification error are amplified by the
map, combine this control with scale-sensitive VC bounds, and
establish lower bounds matching the resulting upper bounds up to
logarithmic factors.

Theorem~\ref{thm:main:abridged} closes the upper--lower gaps in
\eqref{eq:Liu-Zhou-upper-rates} and
\eqref{eq:Liu-Zhou-lower-rates} for the
\(\chi^2\)-divergence, extends the sample-complexity characterization
to every Cressie--Read order \(k>1\), and recovers the classical PAC
rates as \(\rho\downarrow0\). By contrast, the explicit constants in
the bounds of \citet{duchi2021learning} and
\citet{zhou2026rademacher}, and in the agnostic upper bound of
\citet{zhou2023sample}, become unbounded in this limit, while the
upper bounds of \citet{zhou2023sample} retain the agnostic
\(\eps^{-4}\) and realizable \(\eps^{-2}\) dependences.
Table~\ref{tab:rate-comparison} gives a compact summary, while
Section~\ref{sec:compare-prior-detailed} provides a detailed
comparison.

\begin{table}[t]
\centering
\small
\begin{tabular}{@{}lcc@{}}
\hline
& Realizable & Agnostic \\
\hline
Classical PAC
& \(\Theta(d/\eps)\)
& \(\Theta(d/\eps^2)\) \\
Previous robust lower bound, \(k=2\)
& \(\Omega_\rho(d/\eps)\)
& \(\Omega_\rho(d/\eps^2)\) \\
Previous robust upper bound, \(k=2\)
& \(\tilde O_\rho(d/\eps^2)\)
& \(\tilde O_\rho(d/\eps^4)\) \\
This work, \(k>1\)
& \(\Theta_{k,\rho}(d/\eps^{k_\star})\)
& \(\tilde\Theta_{k,\rho}(d/\eps^{k_\star\vee2})\) \\
\hline
\end{tabular}
\caption{Leading \(d\)- and \(\eps\)-dependence at fixed confidence.
For the robust rates, \(\rho>0\) is fixed and
\(\eps\downarrow0\). The previous agnostic lower bound holds for
\(\rho\) below a universal constant. Tildes suppress logarithmic
factors.}
\label{tab:rate-comparison}
\end{table}

\begin{remark}
\emph{Throughout, \(k>1\) is fixed, as our analysis is not uniform as
\(k\downarrow1\); Section~\ref{sec:KL} discusses the endpoint
\(k=1\).}
\end{remark}

\subsubsection{Our Approach: An Overview} 
\label{sec:overview}
We outline the main ideas behind
the upper bounds in Theorem~\ref{thm:main:abridged}. To emphasize the statistical scales,
the resulting \(\eps\)-exponents, and the transition at \(k=2\), we
suppress logarithmic factors and the large-\(\rho\) refinement needed
for the exact dependence on \(\rho\). The full analysis in
Section~\ref{sec:PAC_rates} retains these factors and applies to every
\(\rho\geq0\).

Writing \(\Ped\) for the empirical distribution of \(n\) i.i.d.\
draws from \(P\), we seek to bound the uniform robust-risk deviation
\begin{equation}
\label{eq:UC-rob}
\sup_{h\in\hypoclass}
\left|
\risk{k,\rho}{h}{P}
-
\risk{k,\rho}{h}{\Ped}
\right|.
\end{equation}
For any DRERM \(\erm\), empirical robust-risk optimality implies that
its robust excess risk is at most twice the quantity in
\eqref{eq:UC-rob}.

For \(h\in\hypoclass\), let
\[
A_h
\coloneqq
\left\{
(x,y)\in\X\times\Y:
h(x)\neq y
\right\}
\]
denote the {error set} of $h$ and let
\[
p_h
\coloneqq
P(A_h)
=
\err{h}{P}
\quad\text{and} \quad
\hat p_h
\coloneqq
\Ped(A_h)
=
\err{h}{\Ped}
\]
denote the population and empirical classification errors,
respectively.
Then
\begin{equation}
\label{eq:EI-arise}
\risk{k,\rho}{h}{P}
=
\sup_{Q\in\us_{k,\rho}(P)}Q(A_h).
\end{equation}
The right-hand side is the maximal inflation of the probability of
the error set \(A_h\) over \(\us_{k,\rho}(P)\). By the
event-probability reduction reviewed in
Section~\ref{sec:prior}, it depends on \(h\) only through
\(p_h=P(A_h)\). Denoting the resulting scalar map by
\(p\mapsto\q{k,\rho}{p}\), we have
\[
\risk{k,\rho}{h}{P}
=
\q{k,\rho}{p_h},
\qquad
\risk{k,\rho}{h}{\Ped}
=
\q{k,\rho}{\hat p_h}.
\]
Consequently, with
\[
\Delta_h
\coloneqq
|p_h-\hat p_h|,
\]
the problem in~\eqref{eq:UC-rob} reduces to controlling the event
inflation discrepancy
\[
\left|
\q{k,\rho}{p_h}
-
\q{k,\rho}{\hat p_h}
\right|,
\]
for which we prove
\begin{equation}
\label{eq:two-branch}
\bigl|
\q{k,\rho}{p_h}
-
\q{k,\rho}{\hat p_h}
\bigr|
\lesssim_k
\Delta_h
+
\rho^{1/k}
\begin{cases}
\Delta_h^{1-1/k},
&
p_h\le\Delta_h
\qquad\text{(deviation-dominated)};
\\[4pt]
\displaystyle
\frac{\Delta_h}{p_h^{1/k}},
&
p_h>\Delta_h \qquad\text{(error-dominated)};

\end{cases}
\end{equation}
see Corollary~\ref{cor:reference_scale_increment} for details.

The first term on the right-hand side of
\eqref{eq:two-branch}, \(\Delta_h\), is the traditional
statistical deviation between the true and empirical errors of
\(h\); we refer to it as the \emph{statistical term}. In
particular, as \(\rho\to0\), the robust-risk deviation reduces, up to \(k\)-dependent constants, to this quantity.

The second term, which we call the \emph{robustness term}, is incurred through robustness and has two branches. In the \emph{deviation-dominated regime},
\(
p_h\leq\Delta_h,
\)
it takes the non-Lipschitz form
\(
\rho^{1/k}\Delta_h^{1-1/k},
\)
so deviations of size \(\Delta_h\) are amplified by a factor
\((\rho/\Delta_h)^{1/k}\). In the
\emph{error-dominated regime},
\(
p_h>\Delta_h,
\)
it takes the form
\(
\left({\rho}/{p_h}\right)^{1/k}\Delta_h,
\)
which is linear in \(\Delta_h\) at the fixed error scale \(p_h\).
The two branches coincide at \(p_h=\Delta_h\). Retaining this
two-branch structure is essential for locating the worst case in
the agnostic analysis and obtaining the tight sample-complexity
rates; see Section~\ref{sec:compare-prior-detailed}.

The role of \eqref{eq:two-branch} differs between the realizable and
agnostic settings.
In the realizable setting, every DRERM \(\erm\) satisfies
\(\hat p_{\erm}=0\). Indeed, realizability implies that the minimum
empirical robust risk is zero, while
\(\q{k,\rho}{p}=0\) if and only if \(p=0\). Hence
\[
\Delta_{\erm}
=
|p_{\erm}-\hat p_{\erm}|
=
p_{\erm}.
\]
Thus, the realizable analysis lies at the crossover, and
\eqref{eq:two-branch} gives
\begin{align}
\label{eq:real_inf_map_intro}    
\q{k,\rho}{p_{\erm}} \lesssim_k p_{\erm} + \rho^{1/k}p_{\erm}^{1-1/k}.
\end{align}
A standard realizable VC bound gives
\(p_{\erm}\lesssim d/n\) with high probability. Substituting this
estimate into~\eqref{eq:real_inf_map_intro} yields the upper bound
\begin{align}
\label{eq:real_rate_intro}    
\frac{d}{n}
+
\rho^{1/k}
\left(\frac{d}{n}\right)^{1-1/k}
\end{align}
for the robust risk of DRERM.

The two terms in~\eqref{eq:real_rate_intro} are comparable when
\(\rho\asymp d/n\); above this scale, the robustness term dominates.
Requiring~\eqref{eq:real_rate_intro} to be at most \(\eps\) yields the
realizable sample complexity in
Theorem~\ref{thm:main:abridged}. 

In the agnostic setting, by contrast, \(p_h\) may vary over
\([0,1]\), while the scale-sensitive VC bound
\[
\Delta_h
\lesssim
\sqrt{p_h\frac{d}{n}}
+
\frac{d}{n}
\]
ties \(\Delta_h\) to \(p_h\). Substituting its right-hand side into
the crossover condition \(p_h\asymp\Delta_h\) places the crossover at
the scale
\(
p_h\asymp{d}/{n},
\)
where the VC bound is also of order \(d/n\). Thus, in the
\emph{near-zero} regime \(p_h\lesssim d/n\), the \(d/n\) term
dominates, whereas for \(p_h\asymp1\), the deviation scale is
\((d/n)^{1/2}\). 
Section~\ref{sec:PAC_rates} shows that, after substituting the VC bound into \eqref{eq:two-branch}, the maximization over \(p_h\) is governed by these two scales. Under the bounded-\(\rho\)
simplification used in this overview, they give the two competing contributions
\begin{equation}\label{eq:robust-bounds-rates}
\frac{d}{n}
+
\rho^{1/k}
\left(\frac{d}{n}\right)^{1/k_\star},
\qquad
\left(\frac{d}{n}\right)^{1/2}
+
\rho^{1/k}
\left(\frac{d}{n}\right)^{1/2}.
\end{equation}
The first expression in
\eqref{eq:robust-bounds-rates} arises from the near-zero regime,
whereas the second arises from the bounded-away-from-zero regime.
The sample complexity in the agnostic setting is determined by whichever of these two is larger and the outcome depends on $k$. For $1 < k < 2$ the left term seen in~\eqref{eq:robust-bounds-rates} overtakes the right term when
$
\rho \gtrsim_k \left({d}/{n}\right)^{1-k/2}; 
$
where at the threshold both terms have $(d/n)^{1/2}$ order of magnitude. For $k \geq 2$, the right term seen in~\eqref{eq:robust-bounds-rates} is always the dominant one of the two. Since $d/n\leq \sqrt{d/n}$, combining the different regimes for $k$ we essentially obtain
$$
\bigl|
\q{k,\rho}{p_h}
-
\q{k,\rho}{\hat p_h}
\bigr| \lesssim_k \sqrt{\frac{d}{n}}
+
\rho^{1/k}
\left(\frac{d}{n}\right)^{\frac1{k_\star}\wedge\frac12}.
$$
Solving this simplified bound for \(n\) explains the agnostic
\(\eps\)-exponents and, in particular, the transition at \(k=2\), seen in Theorem~\ref{thm:main:abridged}.
The full analysis in Section~\ref{sec:PAC_rates} removes the
bounded-\(\rho\) simplification used here and recovers the exact
dependence on \(\rho\).

\bigskip
\noindent\emph{Organization.}
Section~\ref{sec:EI} develops the event inflation map and its
discrepancy bounds. Section~\ref{sec:PAC_rates} derives the full
realizable and agnostic sample-complexity results. Proofs are given in
Appendices~\ref{app:EI-proofs} and~\ref{app:PAC_rates}, respectively.
Section~\ref{sec:compare-prior-detailed} compares our results with prior work in more detail.

\section{Event Inflation}
\label{sec:EI}

This section treats event inflation, first under general
\(f\)-divergences, and then under the Cressie--Read family.

\subsection{General Event Inflation}\label{sec:gen-EI}
The distributionally robust risk $\risk{\mathcal U_{f,\rho}}{h}{P}$ is 
the largest possible error of \(h\) under an admissible
perturbation of \(P\).
Equivalently, it quantifies the maximal inflation of the probability
of the error event
\[
A_h=\{(x,y)\in\X\times\Y:h(x)\neq y\},
\qquad h\in\hypoclass.
\]
Thus, for the \(0\)--\(1\)-loss, distributional robustness reduces to understanding how event probabilities inflate under admissible perturbations of the underlying distribution.

More generally, given an event $A\subseteq \X\times\Y$, the {\em event inflation problem} for $A$ with respect to $P$ calls for the evaluation of 
\begin{equation}\label{eq:inflate-def}
\q{f}{A,P,\rho}\coloneqq \sup_{
Q\in\usa{P}{f,\rho}
}Q(A).
\end{equation}

The requirement that \(Q\ll P\) implies that if \(P(A)=0\) then
\(Q(A)=0\) for every \(Q\in\usa{P}{f,\rho}\). Thus, null $P$-events cannot be inflated:
\[
\q{f}{A,P,\rho}=0
\qquad\text{whenever }P(A)=0.
\]
Similarly, if \(P(A)=1\) then \(Q(A)=1\) for every probability measure \(Q\ll P\), and hence
\[
\q{f}{A,P,\rho}=1 \qquad\text{whenever }P(A)=1.
\]
Thus the only nontrivial case is when
\(0<P(A)<1\).

For \(Q\ll P\), let \(L=dQ/dP\). Then
\[
Q(A)=\EE_P[L\ch_A],
\qquad
D_f(Q\|P)=\EE_P[f(L)],
\]
where \(L\geq0\) and \(\EE_P[L]=1\). Hence
\begin{equation}
\label{eq:inflation-L-formulation}
\q{f}{A,P,\rho}
=
\sup_L
\left\{
\EE_P[L\ch_A]:
L\geq0,\ 
\EE_P[L]=1,\ 
\EE_P[f(L)]\leq\rho
\right\}.
\end{equation}
The following proposition records the binary reduction underlying the
event-probability results reviewed in Section~\ref{sec:prior}. Its
proof replaces any feasible likelihood ratio by its conditional
averages on \(A\) and \(A^c\), preserving \(Q(A)\) without increasing
the \(f\)-divergence. We include the short argument for
self-containment.

\begin{proposition}\label{prop:2-value-char}
Given a convex $f$ with $f(1)=0$, a radius $\rho \geq 0$, distributions $P$ and $Q \in \usa{P}{f,\rho}$ with likelihood ratio $L\coloneqq dQ/dP$, as well as a measurable event $A\subseteq \X\times\Y$ with $P(A)\in(0,1)$, the two-valued likelihood ratio
$$
\tilde L(x,y)
 \coloneqq\frac{d\tilde Q}{dP} (x,y) 
= 
\begin{cases}
\EE_P\!\left[L| A\right], & (x,y) \in A;
\\[4pt]
\EE_P\!\left[L| A^c\right], & (x,y) \in A^c,
\end{cases}
$$
induces $\tilde Q \in\usa{P}{f,\rho}$ with $\tilde Q(A)=\EE_P[\tilde L \ch_A] = \EE_P[L \ch_A] = Q(A)$.
\end{proposition}

\begin{proof}
By construction, \(\tilde L\geq0\), and
\[
\EE_P[\tilde L]
=
P(A)\EE_P[L\mid A]
+
P(A^c)\EE_P[L\mid A^c]
=
\EE_P[L]
=
1.
\]
Moreover,
\[
\tilde Q(A)
=
\EE_P[\tilde L\ch_A]
=
P(A)\EE_P[L\mid A]
=
\EE_P[L\ch_A]
=
Q(A).
\]
Finally, Jensen's inequality on \(A\) and \(A^c\) gives
\begin{align*}
\EE_P[f(\tilde L)]
&=
P(A)f\!\left(\EE_P[L\mid A]\right)
+
P(A^c)f\!\left(\EE_P[L\mid A^c]\right)
\\[4pt]
&\leq
P(A)\EE_P[f(L)\mid A]
+
P(A^c)\EE_P[f(L)\mid A^c]
=
\EE_P[f(L)]
\leq
\rho.
\end{align*}
Hence
\(\tilde Q\in\usa{P}{f,\rho}\).
\end{proof}

Let \(p=P(A)\in(0,1)\). By
Proposition~\ref{prop:2-value-char}, it suffices to consider
two-valued likelihood ratios. If \(Q(A)=q\), normalization uniquely
determines such a likelihood ratio as
\[
L
=
\frac{q}{p}\ch_A
+
\frac{1-q}{1-p}\ch_{A^c}.
\]
Since \(Q=P\) is feasible, the supremum may be restricted to
\(q\in[p,1]\). Therefore,
\begin{equation}
\label{eq:1D-optimise}
\q{f,\rho}{p}
\coloneqq
\q{f}{A,P,\rho}
=
\sup\left\{
q\in[p,1]:
p f\left(\frac{q}{p}\right)
+
(1-p)f\left(\frac{1-q}{1-p}\right)
\leq\rho
\right\}.
\end{equation}
Together with
\[
\q{f,\rho}{0}=0,
\qquad
\q{f,\rho}{1}=1,
\]
this defines the event-inflation map
\(p\mapsto\q{f,\rho}{p}\). In particular,
\(\q{f}{A,P,\rho}\) depends on \(A\) and \(P\) only through
\(p=P(A)\).
 
\subsection{Event Inflation: Cressie--Read Family}
\label{sec:EI_CR}

For $k\in(1,\infty)$, let $f_k$ be as in \eqref{eq:CR_intro}.
Given distributions $Q \ll P$, recall that the {\em Cressie--Read divergence of order $k$} of $Q$ with respect to $P$ is given by 
$$
D_k(Q \| P)\coloneqq \EE_P[f_k(L)] = \frac{\EE_P[L^k]-1}{k(k-1)},
$$
where $L = dQ/dP$. 
The constraint $D_k(Q \| P) \leq \rho$ can then be rewritten as the moment bound  
$$\frac{\EE_P[L^k]-1}{k(k-1)} \leq \rho.
$$

\medskip
Setting $f = f_k$ in~\eqref{eq:1D-optimise}, one reaches the identity
\begin{align}
\nonumber
D_k\big(\mathrm{Ber}(q)\|\mathrm{Ber}(p)\big) 
&=
pf_k\left(\frac{q}{p}\right) +(1-p)f_k\left(\frac{1-q}{1-p}\right) 
\\&= 
\frac{1}{k(k-1)} \left(\frac{q^k}{p^{k-1}}+\frac{(1-q)^k}{(1-p)^{k-1}}-1 \right)\,=:\,D_k\big(q\|p\big),
\label{eq:D_k_q_p}  
\end{align}
allowing us to specialize~\eqref{eq:1D-optimise} to the current setting and consequently write 
\begin{align}
\label{eq:q_opt_prob}   
\q{k,\rho}{p}
= \sup \left\{q \in [p,1]: D_k\big(q\|p\big) \leq \rho\right\},
\end{align}
where $p\coloneqq P(A)$.

\medskip
For $\rho>0$ and $k >1$, there is a critical value $0<p_c(\rho;k)<1$ such that $\q{k,\rho}{p}=1$ whenever $p\geq p_c(\rho;k)$.
In other words, if the event $A$ holds with probability $P(A)\geq p_c(\rho;k)$, then the adversary may inflate the probability of $A$ so that $A$ holds with probability one post-inflation. 
Indeed,
$$
\sup_{Q \ll P: D_k(Q\|P) \leq \rho} Q(A) =1 \iff D_k(1\|p) \leq \rho.
$$
The condition $D_k(1\|p) \leq \rho$ is equivalent to 
\begin{align}
p \geq \frac{1}{(1+k(k-1)\rho)^{\frac {1} {k-1}}} =: p_c(\rho;k).
\label{eq:p_c}
\end{align}

For a fixed \(p\in(0,1)\), the map \(q\mapsto D_k(q\|p)\) is continuous and strictly increasing in \(q\in(p,1]\), since
\begin{align*}
\frac{\partial D_k(q\|p)}{\partial q} = \frac{1}{k-1} \left(\frac{q^{k-1}}{p^{k-1}}-\frac{(1-q)^{k-1}}{(1-p)^{k-1}}\right) > 0.
\end{align*}
For \(0<p<p_c(\rho;k)\), \eqref{eq:p_c} gives
\[
D_k(p\|p)=0<\rho<D_k(1\|p).
\]
Since \(q\mapsto D_k(q\|p)\) is continuous and strictly increasing on
\([p,1]\), there is a unique \(q\in(p,1)\) satisfying
\(D_k(q\|p)=\rho\), and this \(q\) is the largest feasible value in
the optimization problem \eqref{eq:q_opt_prob}. Hence,
\begin{align}
\q{k,\rho}{p}
=
\begin{cases}
0, & p=0,
\\[4pt]
\text{the unique } q\in(p,1) \text{ such that } D_k(q\|p)=\rho,
& \text{for } 0<p<p_c(\rho;k),
\\[4pt]
1,
& \text{for } p\ge p_c(\rho;k).
\end{cases}
\label{eq:q(p)_def}
\end{align}

At this point, it is useful to first examine the case \(k=2\), for
which the defining equation can be solved explicitly.  In
Appendix~\ref{sec:eq:q_2_closed_form}, we prove that
\begin{align}
\q{2,\rho}{p}
=
\begin{cases}
\displaystyle
p+\sqrt{2\rho\,p(1-p)},
&
0\le p<p_c(\rho;2),
\\[8pt]
1,
&
p\ge p_c(\rho;2),
\end{cases}
\qquad\text{where}\qquad
p_c(\rho;2)=\frac{1}{1+2\rho}.
\label{eq:q_2_closed_form}
\end{align}
For \(0<p<p_c(\rho;2)\), \eqref{eq:q_2_closed_form} gives
\[
\frac{d\q{2,\rho}{p}}{dp}
=
1+
\frac{\sqrt{2\rho}(1-2p)}
     {2\sqrt{p(1-p)}}.
\]
Thus, as \(p\downarrow0\), the derivative is of order
\(1+\sqrt{\rho/p}\), whereas the map is constant for
\(p\geq p_c(\rho;2)\). Hence, the strongest local amplification
occurs for perturbations of rare-event probabilities.

The following lemma gives the regularity properties and derivative
bounds of \(p\mapsto\q{k,\rho}{p}\) for every \(k>1\).

\begin{lemma}
\label{lem:q(p)_well_defined}
Let \(k>1\) and \(\rho>0\), and let \(p_c(\rho;k)\) be as in
\eqref{eq:p_c}. The map \(p\mapsto\q{k,\rho}{p}\) is
absolutely continuous on \([0,1]\), differentiable on
\((0,p_c(\rho;k))\), and there exist constants
\(
\tilde c_k,\tilde C_k,c_k,C_k>0
\)
such that for every \(p\in(0,p_c(\rho;k))\),
\begin{align}
0
\leq
\tilde C_k
+
\tilde c_k
\left(\frac{\rho}{p}\right)^{1/k}
\leq
\frac{d\q{k,\rho}{p}}{dp}
\leq
C_k
+
c_k
\left(\frac{\rho}{p}\right)^{1/k}.
\label{eq:q'_bound}
\end{align}
\end{lemma}

\begin{remark}
The constants in Lemma~\ref{lem:q(p)_well_defined} depend on \(k\), and
no uniformity is asserted as \(k\downarrow1\).  Throughout the sequel,
\(k>1\) is fixed when invoking the derivative bounds in
\eqref{eq:q'_bound}.
\end{remark}

\bigskip
We next examine the event inflation discrepancy for \(k=2\).
Let
\[
\Delta\coloneqq |p-p'|
\]
and suppose that \(p,p'\in(0,p_c(\rho;2))\). By
\eqref{eq:q_2_closed_form},
\begin{align*}
\left|
\q{2,\rho}{p}
-
\q{2,\rho}{p'}
\right|
&\leq
\Delta
+
\sqrt{2\rho}\,
\frac{\Delta}{
\sqrt{p(1-p)\vee p'(1-p')}
}.
\end{align*}
Indeed,
\[
\left|
\sqrt{p(1-p)}
-
\sqrt{p'(1-p')}
\right|
=
\frac{
|p(1-p)-p'(1-p')|
}{
\sqrt{p(1-p)}+\sqrt{p'(1-p')}
}
\leq
\frac{\Delta}{
\sqrt{p(1-p)\vee p'(1-p')}
}.
\]
By contrast, the cruder inequality
\[
\left|
\sqrt{p(1-p)}
-
\sqrt{p'(1-p')}
\right|
\leq
\sqrt{\Delta}
\]
gives only
\[
\left|
\q{2,\rho}{p}
-
\q{2,\rho}{p'}
\right|
\leq
\Delta+\sqrt{2\rho\Delta},
\]
which loses the dependence on the underlying event-probability scale.

If, in addition, \(p,p'\) are bounded away from \(1\), the preceding upper bound and the lower derivative bound in~\eqref{eq:q'_bound} give
\begin{equation}
\label{eq:k=2-regime-capture}
\left|
\q{2,\rho}{p}
-
\q{2,\rho}{p'}
\right|
\asymp
\Delta
+
\sqrt{\rho}\,
\frac{\Delta}{\sqrt{p\vee p'}}.
\end{equation}
Indeed, the case \(p=p'\) is immediate. Otherwise, after
interchanging \(p\) and \(p'\) if necessary, assume that \(p<p'\).
Then,
\[
\q{2,\rho}{p'}
-
\q{2,\rho}{p}
\gtrsim
\Delta
+
\sqrt{\rho}
\int_p^{p'}t^{-1/2}\,dt
\gtrsim
\Delta
+
\sqrt{\rho}\,
\frac{\Delta}{\sqrt{p'}}.
\]

Consequently, when \(\Delta\ll p\vee p'\), the discrepancy is
linear in \(\Delta\), with amplification factor of order
\[
1+\sqrt{\frac{\rho}{p\vee p'}}.
\]
When \(\Delta\asymp p\vee p'\), it has order
\[
\Delta+\sqrt{\rho\Delta}.
\]
These are the error-dominated and deviation-dominated regimes,
respectively.

For \(p,p'\in[0,1]\), write
\[
\Delta\coloneqq|p'-p|,
\qquad
\cp\coloneqq p\wedge p_c(\rho;k),
\qquad
\cp^{\,\prime}\coloneqq p'\wedge p_c(\rho;k),
\qquad
\cd\coloneqq|\cp^{\,\prime}-\cp|.
\]
Since
\[
\q{k,\rho}{p}
=
\q{k,\rho}{\cp},
\qquad p\in[0,1],
\]
the following theorem gives the corresponding discrepancy bounds for
every \(k>1\).

\begin{theorem}
\label{thm:qp-qp'_bounds}
Let \(k>1\), \(\rho>0\) and let \(p_c(\rho;k)\) be as in
\eqref{eq:p_c}. 
There exist constants $c_k,\tilde c_k>0$ such that, when \(p,p'\in[0,1]\),
\begin{equation}\label{eq:disc-upper}
\bigl|
\q{k,\rho}{p}
-
\q{k,\rho}{p'}
\bigr|
\le
c_k\,\cd\cdot
\left(
1 +
\left(\frac{\rho}
     {\bar p\vee\bar p'}\right)^{1/k}
     \right);
\end{equation}
this with the convention that the right-hand side is zero when
\(\cd=0\).

Moreover, when \(0\leq p<p'\leq p_c(\rho;k)\),
\begin{equation}\label{eq:disc-lower}
\q{k,\rho}{p'}
-
\q{k,\rho}{p}
\ge
\tilde c_k\,\Delta
\cdot\left(
1 +
\left(\frac{\rho}
     {p'}\right)^{1/k}
    \right).
\end{equation}
\end{theorem}

Later on we require a discrepancy bound expressed
only in terms of \(\cp\) and \(\cd\).  To eliminate
\(\cp'\), observe that
\[
\frac{\cd}
     {(\bar p\vee\bar p^{\,\prime})^{1/k}}
\leq
\frac{\cd}
     {(\bar p\vee\cd)^{1/k}}
=
\begin{cases}
\cd^{1-1/k},
&
\cp\le\cd;
\\[4pt]
\displaystyle
\frac{\cd}{\cp^{1/k}},
&
\cp>\cd.
\end{cases}
\]

\begin{corollary}
\label{cor:reference_scale_increment}
Let \(k>1\), \(\rho>0\), and
\(c_k>0\) be as in Theorem~\ref{thm:qp-qp'_bounds}.
Then, 
whenever \(p,p'\in[0,1]\),
\[
\bigl|
\q{k,\rho}{p}
-
\q{k,\rho}{p'}
\bigr|
\le
c_k\cd
+
c_k\rho^{1/k}
\begin{cases}
\cd^{1-1/k},
&
\cp\le\cd
\qquad\textup{(deviation-dominated)};
\\[4pt]
\displaystyle
\frac{\cd}{\cp^{1/k}},
&
\cp>\cd \qquad\textup{(error-dominated)}.

\end{cases}
\]
\end{corollary}

In particular, when $\cp'=0$, one has $\cd=\cp$.
Corollary~\ref{cor:reference_scale_increment} then gives a bound on the inflation map:
\begin{align}
\label{eq:gen-k-regime-capture}
\q{k,\rho}{p} \leq 
c_k\bar{p}\,+\,c_k\rho^{1/k}\,\bar{p}^{1-\frac{1}{k}},
\qquad p\in[0,1].
\end{align}

\begin{remark}[The Degenerate Radius \(\rho=0\)]
\emph{For \(\rho=0\), we use
\(\q{k,0}{p}=p\) for \(p\in[0,1]\), which follows from
\(\usa{P}{k,0}=\{P\}\).}
\end{remark}

\section{Sample Complexity Analysis}
\label{sec:PAC_rates}
This section states the realizable and agnostic sample-complexity
bounds underlying Theorem~\ref{thm:main:abridged}. All proofs are
deferred to Appendix~\ref{app:PAC_rates}.

\subsection{Uniform Robust-Risk Deviation}
\label{sec:pre}

For a hypothesis class \(\hypoclass\) satisfying
\(\VC(\hypoclass)=d\), our derivation of the sample-complexity upper
bounds proceeds by establishing high-probability control of the uniform
robust-risk deviation
\[
\sup_{h\in\hypoclass}
\left|
\risk{k,\rho}{h}{P}
-
\risk{k,\rho}{h}{\Ped}
\right|.
\]

For $h\in\hypoclass$, write
\[
p_h \coloneqq  \err{h}{P}
\qquad\text{and}\qquad 
\hat p_h \coloneqq  \err{h}{\Ped}.
\]
Following Section~\ref{sec:EI}, the relevant quantities are the clipped error levels
\[
\cp_h\coloneqq  p_h\wedge p_c(\rho;k),
\qquad
\cep_h\coloneqq  \hat p_h\wedge p_c(\rho;k),
\qquad
\cd_h
\coloneqq 
|\cep_h-\cp_h|,
\]
where the clipping error level $p_c(\rho;k)$ is as in \eqref{eq:p_c}.
Following Section~\ref{sec:EI}, we may write
\[
\bigl|
\risk{k,\rho}{h}{P}
-
\risk{k,\rho}{h}{\Ped}
\bigr|
=
\bigl|
\q{k,\rho}{\cp_h}
-
\q{k,\rho}{\cep_h}
\bigr|.
\]
Corollary~\ref{cor:reference_scale_increment} then asserts 
\begin{align}
\label{eq:desc_bound_agnostic_discussion}   
\bigl|
\risk{k,\rho}{h}{P}
-
\risk{k,\rho}{h}{\Ped}
\bigr|
\lesssim_k
\cd_h
+
\rho^{1/k}
\min\left\{
\cd_h^{1-1/k},
\frac{\cd_h}{\cp_h^{1/k}}
\right\}.
\end{align}
This leads to the robust-risk deviation being uniformly bounded by
\begin{equation}\label{eq:1D-reduction}
\sup_{h\in\hypoclass}\,
\bigl|
\risk{k,\rho}{h}{P}
-
\risk{k,\rho}{h}{\Ped}
\bigr| 
\lesssim_k \sup_{h\in\hypoclass}
\left[
\cd_h
+
\rho^{1/k}
\min\left\{
\cd_h^{1-1/k},
\frac{\cd_h}{\cp_h^{1/k}}
\right\}
\right].
\end{equation}
Once \(\cd_h\) is bounded in terms of \(\cp_h\), the right-hand side
of~\eqref{eq:1D-reduction} is controlled by a one-dimensional
maximization over
\(\cp_h\in[0,p_c(\rho;k)]\).

The following lemma provides the scale-sensitive VC control needed
below. It follows from the classical relative VC bounds
of~\citet[Theorem~5.1]{boucheron2005theory}, the VC growth bound, and
the clipping argument given in Appendix~\ref{app:clipped-vc-master}.
For \(d\geq1\), \(n\geq d\), and \(\delta\in(0,1)\), define
\begin{equation}
\label{eq:Gamma_def}
\Gamma_{d,n}(\delta)
\coloneqq
d\log\left(\frac{2\e n}{d}\right)
+
\log\left(\frac{8}{\delta}\right),
\end{equation}
and abbreviate \(\Gamma_{d,n}(\delta)\) by \(\GammaMac\) when no
confusion can arise.

\begin{lemma}[Clipped scale-sensitive VC bound]
\label{lem:clipped-vc-master}
Fix \(k>1\) and \(\rho\geq0\). Let \(\hypoclass\) be a hypothesis
class with
\(
\VC(\hypoclass)=d\geq1,
\)
let \(n\geq d\), and let \(\delta\in(0,1)\).
There exists a universal constant \(C>0\) such that, with
probability at least \(1-\delta\),
\begin{align}
\label{eq:master_vc_abs_clip}
\cd_h
\leq
C\left(
\sqrt{\frac{\cp_h\GammaMac}{n}}
+
\frac{\GammaMac}{n}
\right)
\end{align}
holds uniformly over all \(h\in\hypoclass\). In particular,
\begin{align}
\label{eq:master_vc_real_clip}
\cep_h=0
\quad\Longrightarrow\quad
\cp_h
\leq
C\,\frac{\GammaMac}{n}
\end{align}
holds uniformly over all \(h\in\hypoclass\).
\end{lemma}

\subsection{Realizable Case}

In the realizable setting, the empirical minimizer $\erm$ satisfies
$\cep_{\erm}=0$, so that $\cd_{\erm}=\cp_{\erm}$. Consequently,
\(
\q{k,\rho}{\cep_{\erm}}=\q{k,\rho}{0}=0,
\)
allowing for
\begin{equation}\label{eq:real-infaltion}
\bigl|
\q{k,\rho}{\cp_{\erm}}
-
\q{k,\rho}{\cep_{\erm}}
\bigr|
=
\q{k,\rho}{\cp_{\erm}}
\overset{\eqref{eq:gen-k-regime-capture}}{\lesssim_k}
\cp_{\erm}+
\rho^{1/k}\cp_{\erm}^{1-1/k}.
\end{equation}
Since the last bound increases with $\cp_{\erm}$, controlling the robust
risk reduces to bounding $\cp_{\erm}$. A bound on $\cp_{\erm}$ is provided by \eqref{eq:master_vc_real_clip} of
Lemma~\ref{lem:clipped-vc-master}, 
asserting that 
\begin{equation}\label{eq:minimiser-clip}
\cep_h=0
\quad\Longrightarrow\quad
\cp_h\lesssim \frac{\GammaMac}{n}
\end{equation}
holds with probability at least $1-\delta$ uniformly over all $h\in\hypoclass$.
For the empirical minimizer, we may thus write
\[
\risk{k,\rho}{\erm}{P}=\q{k,\rho}{\cp_{\erm}}
\lesssim_k
\frac{\GammaMac}{n}
+
\rho^{1/k}
\left(\frac{\GammaMac}{n}\right)^{1-1/k},
\]
owing to~\eqref{eq:real-infaltion} and~\eqref{eq:minimiser-clip}.
Requiring the right-hand side of the preceding inequality to be at
most \(\eps\) yields the sufficient condition
\begin{align}
\label{eq:samp_suff_cond_real}
n
\gtrsim_k
\Psi_{k,\rho}^{\mathrm{real}}(\eps)\cdot\GammaMac,
\end{align}
where
\begin{equation}
\label{eq:Psi-real-def}
\Psi_{k,\rho}^{\mathrm{real}}(\eps)
\coloneqq
\frac{1}{\eps}
+
\frac{\rho^{1/(k-1)}}{\eps^{k_\star}},
\qquad
k_\star\coloneqq\frac{k}{k-1}.
\end{equation}
Since \(\eps\in(0,1)\),
\(\Psi_{k,\rho}^{\mathrm{real}}(\eps)\geq1\). Moreover,
\(\GammaMac\) depends on \(n\), so
\eqref{eq:samp_suff_cond_real} is an implicit sample-size condition.
After absorbing \(k\)-dependent constants,
Lemma~\ref{lem:self-consistency} in
Appendix~\ref{sec:self-cons} shows that this condition is implied by
\[
n
\gtrsim_k
\Psi_{k,\rho}^{\mathrm{real}}(\eps)
\left(
d\log\left(
\e\Psi_{k,\rho}^{\mathrm{real}}(\eps)
\right)
+
\log\delta^{-1}
\right).
\]

This is the DRERM sample-size condition in the next theorem. 
The theorem also records a log-free upper bound for the realizable
sample complexity when the learning rule is not required to be proper.
This second bound does not follow from
the ERM or DRERM analysis; it is obtained by transferring the optimal
classical realizable learner of
\citet[Theorem~2]{hanneke2016optimal} through the event-inflation
map. That learner outputs a majority vote which need not belong to
\(\hypoclass\), and hence need not be an ERM or DRERM.

\medskip
\begin{theorem}[Realizable upper bounds]
\label{thm:main_upper_realizable}
Let \(k>1\), \(\rho\geq0\), and let \(\hypoclass\) be a hypothesis
class with \(\VC(\hypoclass)=d\geq1\). There exists a constant
\(C_k>0\) such that the following statements hold for every
\(\eps,\delta\in(0,1)\).

For every distribution \(P\) realizable with respect to
\(\hypoclass\), every DRERM \(\erm\) as in \eqref{eq:DRERM}
satisfies
\[
\Pr_{S_n\sim P^n}
\left[
\risk{k,\rho}{\erm}{P}\leq\eps
\right]
\geq1-\delta
\]
whenever
\[
n
\geq
C_k\,
\Psi_{k,\rho}^{\mathrm{real}}(\eps)
\left(
d\log\left(
\e \Psi_{k,\rho}^{\mathrm{real}}(\eps)
\right)
+
\log(1/\delta)
\right).
\]

Moreover, the realizable sample complexity satisfies
\[
\sampcomp_{k,\rho}^{\real}(\eps,\delta;\hypoclass)
\leq
C_k\,
\Psi_{k,\rho}^{\mathrm{real}}(\eps)
\bigl(d+\log(1/\delta)\bigr).
\]
\end{theorem}

\bigskip
The realizable proxy can be equivalently written as
\[
\Psi_{k,\rho}^{\mathrm{real}}(\eps)
=
\frac{1}{\eps}
\left(
1+
\left(\frac{\rho}{\eps}\right)^{1/(k-1)}
\right).
\]
This form makes explicit that, for fixed \(k>1\), the realizable rate is
governed by the ratio \(\rho/\eps\).  If
\(\rho/\eps\lesssim_k1\), then
\(
\Psi_{k,\rho}^{\mathrm{real}}(\eps)
\asymp_k
1/\eps,
\)
and the classical realizable \(\eps^{-1}\) dependence is recovered.  If
\(\rho/\eps\gtrsim_k1\), then the robust branch dominates and
\(
\Psi_{k,\rho}^{\mathrm{real}}(\eps)
\asymp_k
{\rho^{1/(k-1)}}/{\eps^{k_\star}}.
\)
Thus, for every fixed \(\rho>0\) and \(k>1\), the robust realizable rate
has \(\eps\)-exponent \(k_\star=k/(k-1)\) as \(\eps\downarrow0\).

\medskip
The following theorem records both the matching minimax lower bound,
which applies even when the learning rule is not required to be
proper, and the stronger lower bound that holds for proper learning
on suitable VC classes.

\begin{theorem}[Realizable lower bounds]
\label{thm:main_lower_realizable}
There exist universal constants \(\delta_0,\eps_0>0\) such that, for
every \(k>1\), there exists a constant \(\tilde c_k>0\) for which
the following holds. Fix
\[
\rho\geq0,
\qquad
\delta\in(0,\delta_0],
\qquad
\eps\in(0,\eps_0].
\]

Let \(\hypoclass\) be a binary hypothesis class with
\[
\VC(\hypoclass)=d<\infty
\qquad\text{and}\qquad
|\hypoclass|\geq3.
\]
Any learning rule
\[
r_n:(\X\times\Y)^n\to\mathcal G
\]
satisfying
\[
\Pr_{S_n\sim P^n}
\left[
\risk{k,\rho}{r_n(S_n)}{P}\leq\eps
\right]
\geq1-\delta
\]
for every distribution \(P\) realizable with respect to
\(\hypoclass\) must satisfy
\[
n
\geq
\tilde c_k\,
\Psi_{k,\rho}^{\mathrm{real}}(\eps)
\bigl(d+\log(1/\delta)\bigr).
\]

Moreover, for every integer \(d\geq1\), there exist a measurable
instance space \(\X_d\) and a binary hypothesis class
\(\hypoclass_d\) of measurable classifiers from \(\X_d\) to \(\Y\),
with \(\VC(\hypoclass_d)=d\), such that every learning rule that is
proper with respect to \(\hypoclass_d\) and satisfies the preceding
guarantee, interpreted on \(\X_d\) with
\(\hypoclass=\hypoclass_d\), must satisfy
\[
n
\geq
\tilde c_k\,
\Psi_{k,\rho}^{\mathrm{real}}(\eps)
\left[
d\log\left(
\e\Psi_{k,\rho}^{\mathrm{real}}(\eps)
\right)
+
\log(1/\delta)
\right].
\]
\end{theorem}

\subsection{Agnostic Case}

In the realizable setting, one deals with an empirical minimizer
\(\erm\in\hypoclass\) for which
\(\cd_{\erm}=\cp_{\erm}\).  Hence, the robust-risk
deviation is controlled by the single quantity
\(\cp_{\erm}\) passed through the inflation map $p\mapsto\q{k,\rho}{p}$ at the error scale $\GammaMac/n$.
In the agnostic setting, however, one considers
hypotheses \(h\in\hypoclass\) for which \(\cp_h\) may range over
\([0,p_c(\rho;k)]\), while \(\cd_h\) is tied to \(\cp_h\) through Lemma~\ref{lem:clipped-vc-master}.
Consequently, in the agnostic setting, one needs to take into account the full discrepancy bound
in the right-hand side of \eqref{eq:desc_bound_agnostic_discussion}, which can be equivalently written as
\[
\cd_h
\left(
1+
\left(
\frac{\rho}{\max\{\cd_h,\cp_h\}}
\right)^{1/k}
\right),
\]
and therefore has two branches, determined by the relative magnitudes
of \(\cp_h\) and \(\cd_h\).  The crossover
\(\cp_h\asymp\cd_h\) separates the nonlinear deviation-dominated
regime, characterized by \(\cp_h\lesssim\cd_h\), from the
error-dominated regime, characterized by
\(\cp_h\gtrsim\cd_h\).  In the former, the \(\rho\)-dependent term
has the form
\[
\rho^{1/k}\cd_h^{1-1/k},
\]
whereas in the latter it has the form
\[
\rho^{1/k}\frac{\cd_h}{\cp_h^{1/k}}
=
\left(
\frac{\rho}{\cp_h}
\right)^{1/k}\cd_h,
\]
and is therefore linear in \(\cd_h\) when \(\cp_h\) is held fixed.

The statistical clipped discrepancy term $\cd_h$ is controlled through the scale-sensitive VC bound
\eqref{eq:master_vc_abs_clip}, which asserts that with probability \(1-\delta\) over the choice of the sample, 
\begin{align}
\label{eq:scale_VC_agnostic_discussion} 
\cd_h
\lesssim
\sqrt{\frac{\cp_h\GammaMac}{n}}
+
\frac{\GammaMac}{n}
\asymp\max\left\{\sqrt{\frac{\cp_h\GammaMac}{n}}
,
\frac{\GammaMac}{n}
\right\}
\end{align}
holds uniformly over all \(h\in\hypoclass\).
The error scale \(\GammaMac/n\) identifies the crossover between the deviation-dominated regime and the error-dominated regime in the right-hand side of \eqref{eq:desc_bound_agnostic_discussion}. Indeed, equating \(\cp_h\) with the right-hand side of
\eqref{eq:scale_VC_agnostic_discussion} yields this order.
Plugging \eqref{eq:scale_VC_agnostic_discussion} into \eqref{eq:desc_bound_agnostic_discussion} we obtain that under the event of \eqref{eq:scale_VC_agnostic_discussion},
simultaneously for all $h\in\hypoclass$,
\begin{equation}
\label{eq:rob-risk-regime-capture-agn}
\bigl|
\q{k,\rho}{\cp_h}
-
\q{k,\rho}{\cep_h}
\bigr|
\,\lesssim_k\,
\begin{cases}
\displaystyle
\frac{\GammaMac}{n}
+
\rho^{1/k}
\left(\frac{\GammaMac}{n}\right)^{1-1/k},
& \, \cp_h \le \dfrac{\GammaMac}{n} 
\quad \text{(deviation-dominated)};
\\[12pt]
\displaystyle
\sqrt{\frac{\cp_h\GammaMac}{n}}
+
\rho^{1/k}
\sqrt{\frac{\GammaMac}{n}}
\,\cp_h^{\frac12-\frac1k},
& \,\cp_h > \dfrac{\GammaMac}{n}
\quad\text{(error-dominated)}
.
\end{cases}
\end{equation}
To see~\eqref{eq:rob-risk-regime-capture-agn}, note that if
\(
\cp_h \le {\GammaMac}/{n}
\),
then
\[
\cd_h
\lesssim
\frac{\GammaMac}{n} \implies \cd_h^{1-\frac1k}
\lesssim
\left(\frac{\GammaMac}{n}\right)^{1-\frac1k}.
\] 
If, instead, 
\(
\cp_h > {\GammaMac}/{n},
\)
then
\[
\cd_h
\lesssim
\sqrt{\frac{\cp_h\GammaMac}{n}} \implies \frac{\cd_h}{\cp_h^{1/k}}
\lesssim
\sqrt{\frac{\GammaMac}{n}}\,
\cp_h^{\frac12-\frac1k}.
\]
The worst case, corresponding to the supremum over $h$ in \eqref{eq:1D-reduction}, is therefore obtained through maximizing \eqref{eq:rob-risk-regime-capture-agn} over
$\cp_h\in[0,p_c(\rho;k)]$, and the outcome depends on $k$.

We next maximize the second branch in
\eqref{eq:rob-risk-regime-capture-agn} over the possible values of
\(\cp_h\) and compare its contribution to the first branch in
\eqref{eq:rob-risk-regime-capture-agn} to determine the worst case.  
If \(\GammaMac/n>p_c(\rho;k)\), then every
\(\cp_h\in[0,p_c(\rho;k)]\) satisfies
\(
\cp_h\le{\GammaMac}/{n}.
\)
Hence only the first branch of
\eqref{eq:rob-risk-regime-capture-agn} applies, and the uniform
robust-risk deviation is bounded by
\[
\frac{\GammaMac}{n}
+
\rho^{1/k}
\left(
\frac{\GammaMac}{n}
\right)^{1-1/k}.
\]
Assume therefore that
\(\GammaMac/n\le p_c(\rho;k)\).
In this regime, the expression to be maximized is
\begin{align}
\label{eq:agnostic_upper_envelope}
\sqrt{\frac{\cp_h\GammaMac}{n}}
+
\rho^{1/k}
\sqrt{\frac{\GammaMac}{n}}\,
\cp_h^{\frac12-\frac1k}.
\end{align}
We claim that the expression in
\eqref{eq:agnostic_upper_envelope} has no interior maximum on
\(
\left[
{\GammaMac}/{n},
p_c(\rho;k)
\right].
\)
Indeed, its derivative with respect to \(\cp_h\) is
\[
\sqrt{\frac{\GammaMac}{n}}\,
\cp_h^{-\frac12-\frac1k}
\left(
\frac{\cp_h^{1/k}}{2}
+
\frac{k-2}{2k}\rho^{1/k}
\right).
\]
For \(k\ge2\), this derivative is nonnegative, so the expression is
nondecreasing.  For \(1<k<2\), the term in parentheses is increasing
in \(\cp_h\), and therefore the derivative can change sign at most
once, and only from negative to positive.  Consequently, the expression
is either monotone on the interval or first decreases and then
increases.  Its maximum is therefore attained at one of the two
endpoints.

Substituting the two endpoints into
\eqref{eq:agnostic_upper_envelope} gives the two competing contributions
\begin{align}
\label{eq:competing_contributions_agnostic}
\frac{\GammaMac}{n}
+
\rho^{1/k}
\left(\frac{\GammaMac}{n}\right)^{1-1/k}
\qquad\text{and}\qquad
\left(
p_c(\rho;k)^{1/2}
+
\rho^{1/k}p_c(\rho;k)^{1/2-1/k}
\right)
\sqrt{\frac{\GammaMac}{n}} .
\end{align}
To handle the second contribution in
\eqref{eq:competing_contributions_agnostic}, we use the following
endpoint comparison lemma.  Its proof is given in
Appendix~\ref{sec:lem:agnostic_endpoint_comparison}.

\begin{lemma}[Endpoint comparison]
\label{lem:agnostic_endpoint_comparison}
Let \(k>1\).
Then, for every \(\rho\ge0\),
\begin{align}
p_c(\rho;k)+\rho^{\frac1k} p_c(\rho;k)^{1-\frac1k}
&\asymp_k 1
\\[4pt]
\label{eq:second_end_point_comparison}  p_c(\rho;k)^{\frac12}+\rho^{\frac1k} p_c(\rho;k)^{\frac12-\frac1k}
&\asymp_k (1+\rho^{\frac{1}{k-1}})^{\frac12}.
\end{align}
\end{lemma} 

Substituting \eqref{eq:second_end_point_comparison} into the second
contribution in \eqref{eq:competing_contributions_agnostic} and
solving the following inequalities for \(n\),
\begin{align}
\label{eq:competing_contributions} 
\frac{\GammaMac}{n}
+
\rho^{1/k}
\left(\frac{\GammaMac}{n}\right)^{1-1/k} \leq \eps 
\qquad\text{and}\qquad
(1+\rho^{\frac{1}{k-1}})^{\frac12}\sqrt{\frac{\GammaMac}{n}} \leq \eps,
\end{align}
yields the sufficient condition 
\begin{align}
\label{eq:suff_cond_agn}    
n\gtrsim_k\GammaMac \cdot\Psi_{k,\rho}^{\agn}(\eps),
\end{align}
where
\begin{align}
\label{eq:Psi_agn}    
\Psi_{k,\rho}^{\agn}(\eps)
&\coloneqq 
\max\left\{
\frac1\eps+\frac{\rho^{1/(k-1)}}{\eps^{k_\star}},
\;
{\frac{
1+\rho^{1/(k-1)}}{\eps^2}
}
\right\}.
\end{align}
The first term in
\(\Psi_{k,\rho}^{\agn}(\eps)\) is the condition obtained from the
left inequality in~\eqref{eq:competing_contributions} and is identical to $\Psi_{k,\rho}^{\real}(\eps)
$, while the
second is obtained from the right inequality.

Since \(\GammaMac\) depends on \(n\),
\eqref{eq:suff_cond_agn} is an implicit sample-size condition.
For \(\eps\in(0,1)\), we have
\(\Psi_{k,\rho}^{\agn}(\eps)\geq1\). Therefore,
Lemma~\ref{lem:self-consistency}, after absorbing \(k\)-dependent
constants, shows that \eqref{eq:suff_cond_agn} is implied by
\[
n
\gtrsim_k
\Psi_{k,\rho}^{\agn}(\eps)
\left(
d\log\left(
\e \Psi_{k,\rho}^{\agn}(\eps)
\right)
+
\log\delta^{-1}
\right).
\]
This is the explicit sample-size condition stated in
Theorem~\ref{thm:main_upper}.

\begin{theorem}[Agnostic upper bound for DRERM]
\label{thm:main_upper}
Let \(k>1\), \(\rho\ge0\) as well as a hypothesis class \(\hypoclass\) with \(
\VC(\hypoclass)=d\geq 1
\). Then, there exists a constant $C_k>0$ such that 
\[
\risk{k,\rho}{\erm}{P}
-
\inf_{h\in\hypoclass}\risk{k,\rho}{h}{P}
\leq
\eps
\]
holds with probability at least $1-\delta$, whenever \(\eps,\delta\in(0,1)\), \(P\) is any distribution over \(\X\times\Y\), 
$\erm$ is a DRERM as in \eqref{eq:DRERM}, and
\begin{equation}\label{eq:agn-rate-unabridged}
n
\ge
C_k\cdot
\Psi_{k,\rho}^{\agn}(\eps) \cdot 
\left(
d\cdot\log\left(\e \Psi_{k,\rho}^{\agn}(\eps)\right)
+
\log\delta^{-1}\right).
\end{equation}
\end{theorem}

\bigskip
For \(\eps\in(0,1)\), the proxy in~\eqref{eq:Psi_agn}
satisfies
\begin{align}
\label{eq:Psi-agn-abridged-comparison}
\Psi_{k,\rho}^{\agn}(\eps)
\asymp
\max\left\{
\frac1{\eps^2},
\frac{\rho^{1/(k-1)}}{\eps^{2\vee k_\star}}
\right\}.
\end{align}
Indeed, \(1/\eps\leq1/\eps^2\), while
\[
\max\left\{
\frac{\rho^{1/(k-1)}}{\eps^{k_\star}},
\frac{\rho^{1/(k-1)}}{\eps^2}
\right\}
=
\frac{\rho^{1/(k-1)}}{\eps^{2\vee k_\star}}.
\]
This is the agnostic rate stated in
Theorem~\ref{thm:main:abridged}.

The two \(\rho\)-dependent terms in~\eqref{eq:Psi_agn}
have different origins. The term
\[
\frac{\rho^{1/(k-1)}}{\eps^{k_\star}}
\]
comes from the crossover
\[
\cp_h
\asymp
\cd_h
\asymp
\frac{\GammaMac}{n},
\]
whereas
\[
\frac{\rho^{1/(k-1)}}{\eps^2}
\]
comes from error levels approaching \(p_c(\rho;k)\), near the
upper end of the error-dominated range.

The classical contribution \(1/\eps^2\) dominates whenever
\begin{align}
\label{eq:rho_threshold}    
\rho
\lesssim
\begin{cases}
\eps^{2-k},
&
1<k<2,
\\[3pt]
1,
&
k\geq2.
\end{cases}
\end{align}
In this regime,
\(
p_c(\rho;k)\asymp_k1
\),
and the rate is supplied by error levels bounded away from zero,
as in classical agnostic PAC learning.

In the complementary regime, robustness changes the rate, and the
location of the rate-determining contribution follows from comparing
\(k_\star\) with \(2\).
If \(1<k<2\), then \(k_\star>2\), and
the crossover contribution dominates. At \(k=2\), the two
\(\rho\)-dependent terms have the same \(\eps\)-dependence and
\[
\Psi_{2,\rho}^{\agn}(\eps)
\asymp
\frac{1+\rho}{\eps^2}.
\]
Moreover, the \(\rho\)-dependent part of
\eqref{eq:agnostic_upper_envelope} is then independent of
\(\cp_h\), so it is produced throughout the error-dominated
range, although the full envelope is maximized near its upper
endpoint. Finally, if \(k>2\), then \(k_\star<2\), and the
contribution from error levels approaching \(p_c(\rho;k)\)
dominates.

Thus, \(k=2\) marks a transition in the localization of the
robustness-dependent contribution: for \(1<k<2\), it comes from
the near-zero crossover; at \(k=2\), it is present throughout
the error-dominated range; and for \(k>2\), it comes from error
levels approaching \(p_c(\rho;k)\). This is also the
large-\(\rho\) refinement suppressed in
Section~\ref{sec:overview}: the overview captures the statistical
scales, the \(\eps\)-exponents, and the transition at \(k=2\),
while the endpoint comparison is needed for the exact
\(\rho\)-dependence when \(k>2\).

\medskip
The two terms in $\Psi_{k,\rho}^{\agn}(\eps)$ of \eqref{eq:Psi_agn} arise from different lower-bound constructions.
The first is inherited from the realizable lower bound, since any agnostic learner must also succeed on realizable distributions.  For the second, the proof uses distributions whose optimal ordinary error is a constant fraction of \(p_c(\rho;k)\).  Passing the corresponding ordinary excess-error lower bounds through the event-inflation map yields the second term in $\Psi_{k,\rho}^{\agn}(\eps)$.  The details are given in Section~\ref{sec:thm:main_lower_agnostic}.

\begin{theorem}[Agnostic lower bound]
\label{thm:main_lower_agnostic}
There exists a universal constant \(\delta_0>0\) such that, for
every \(k>1\), there exist constants
\(\tilde c_k,\tilde\eps_k>0\) for which
the following holds.
Let \(\rho\ge0\), let \(\delta\in(0,\delta_0]\), and let
\(\hypoclass\) be a hypothesis class with
\[
\VC(\hypoclass)=d<\infty
\qquad\text{and}\qquad
|\hypoclass|\geq3.
\]
For every \(\eps\in(0,\tilde\eps_k]\), any learning rule
\(
r_n:(\X\times\Y)^n\to\mathcal G
\)
satisfying, for every distribution \(P\) over \(\X\times\Y\),
\[
\Pr_{S_n\sim P^n}\!\left[
\risk{k,\rho}{r_n(S_n)}{P}
-
\inf_{h\in\hypoclass}\risk{k,\rho}{h}{P}
\leq \eps
\right]
\geq 1-\delta
\]
must satisfy
\[
n
\ge
\tilde c_k\,
\Psi_{k,\rho}^{\agn}(\eps)
\bigl(d+\log(1/\delta)\bigr).
\]
\end{theorem}

\section{Comparison with Prior Work and the Case of KL}
\label{sec:compare-prior-detailed}
The analyses of distributionally robust learning under
Cressie--Read divergences in
\citet{duchi2021learning},
\citet{zhou2023sample}, and
\citet{zhou2026rademacher}
all start from Shapiro's dual formulation
\citep{shapiro2017distributionally}.
Translating Shapiro's terminology into our own, 
set
\[
\ell_h(x,y)\coloneqq \ind{h(x)\neq y},
\qquad
c_k(\rho)\coloneqq (1+k(k-1)\rho)^{1/k};
\]
Shapiro duality characterizes robust risk through 
\begin{equation}\label{eq:Shapiro}
\risk{k,\rho}{h}{P}
=
\inf_{\eta\in\R}
\left\{
g_k(\eta,h;P)
\coloneqq 
c_k(\rho)\,
\Bigl(
\EE_{(X,Y)\sim P}
\bigl[(\ell_h(X,Y)-\eta)_+^{k_\star}\bigr]
\Bigr)^{1/k_\star}
+\eta
\right\}.
\end{equation}
The relevant distinction is whether the \(p_h\)-dependent
two-branch event inflation discrepancy is retained when statistical
bounds are applied. For the \(\chi^2\)-divergence, the upper bounds of~\citet{zhou2023sample} are consistent with the coarser
global estimate
\[
\q{2,\rho}{p}
\leq
p+\sqrt{2\rho\,p}
\leq
(1+\sqrt{2\rho})\sqrt p,
\qquad p\in[0,1],
\]
which collapses the two branches in
\eqref{eq:k=2-regime-capture} into a single square-root bound.
The dependence on the underlying error scale is then lost, and the
resulting upper bounds have \(\eps^{-4}\) agnostic and
\(\eps^{-2}\) realizable dependence.

The later work of~\citet{zhou2026rademacher} likewise does not retain the
\(p_h\)-dependent two-branch structure in the statistical step and
obtains the excess-risk rate \(n^{-(k-1)/(2k)}\). By retaining this
structure before combining it with the scale-sensitive VC bound, our
analysis identifies the rate-determining error scale and the change
in its location at \(k=2\).

A separate issue is the behavior as \(\rho\to0\).
The explicit finite-sample bounds of
\citet{duchi2021learning} and \citet{zhou2026rademacher}, as well as
the agnostic upper bound of \citet{zhou2023sample}, become unbounded
in this limit.
\citet{duchi2021learning} identify the robust-risk estimation rate
\(
n^{-1/(k_\star\vee2)}
\)
and the mechanism by which robustness forces the estimation of
higher-order tail quantities. Their argument, however, controls,
uniformly over an admissible dual interval \(I_\rho\),
\[
\sup_h\sup_{\eta\in I_\rho}
\left|
g_k(\eta,h;P_n)-g_k(\eta,h;P)
\right|.
\]
For bounded losses, the length of this interval is of order
\((c_k(\rho)-1)^{-1}\), which diverges as \(\rho\to0\) and causes the
corresponding blow-up.

For the \(0\)--\(1\)-loss, the Bernoulli structure instead resolves
the adversarial optimization completely through
\(
\risk{k,\rho}{h}{P}
=
\q{k,\rho}{p_h}\)
and \(
p_h=\err{h}{P}.
\)
This avoids uniform control over a dual interval whose length
deteriorates as \(\rho\to0\), allowing our bounds to recover the
classical PAC rates in the non-robust limit.

\subsection{The Case of KL}
\label{sec:KL}
The present analysis treats fixed Cressie--Read orders \(k>1\).
The endpoint \(k=1\), corresponding to the KL divergence,
appears to be qualitatively different rather than merely a
limiting case of the results above.  Indeed, for
\[
D_{\mathrm{KL}}(Q\|P)
=
\int \log\!\left(\frac{dQ}{dP}\right)\,dQ,
\]
the maximal inflation of an event \(A\) with \(P(A)=p\) is
the largest \(q\in[p,1]\) satisfying the binary constraint
\[
q\log\frac q p
+
(1-q)\log\frac{1-q}{1-p}
\le \rho .
\]

Fix \(\rho>0\), and let
\[
q_p\coloneqq\q{\mathrm{KL},\rho}{p}
\]
denote this maximal value. For all sufficiently small \(p\),
\(\log(1/p)>\rho\), so \(q_p\in(p,1)\) and the binary constraint is
active:
\[
q_p\log\frac{q_p}{p}
+
(1-q_p)\log\frac{1-q_p}{1-p}
=
\rho.
\]
Necessarily \(q_p\to0\) as \(p\downarrow0\); otherwise, along a
subsequence bounded away from zero, the first term on the left-hand
side would diverge while the second remains bounded below. Expanding
the equality gives
\[
\rho
=
q_p\log\frac1p
+
q_p\log q_p
+
(1-q_p)\log(1-q_p)
-
(1-q_p)\log(1-p).
\]
Since \(p\to0\) and \(q_p\to0\), the last three terms are \(o(1)\).
Consequently,
\[
q_p\log\frac1p
=
\rho+o(1),
\]
and hence
\[
\q{\mathrm{KL},\rho}{p}
\sim
\frac{\rho}{\log(1/p)}
\qquad\text{as }p\downarrow0.
\]

This should be contrasted with the Cressie--Read case of fixed
\(k>1\) and \(\rho>0\). There the corresponding binary constraint is
\[
\frac{1}{k(k-1)}
\left[
q^k p^{1-k}
+
(1-q)^k(1-p)^{1-k}
-1
\right]
\leq \rho,
\]
and, as \(p\downarrow0\), the maximal feasible \(q\) has order
\[
\q{k,\rho}{p}
\asymp_{k,\rho}
p^{1-1/k}.
\]
Thus, for every fixed \(k>1\) and \(\rho>0\),
\[
\q{k,\rho}{p}
\ll
\q{\mathrm{KL},\rho}{p},
\qquad p\downarrow0.
\]
In this sense, KL balls permit a much stronger amplification
of rare events than any fixed Cressie--Read ball of order
\(k>1\).  This suggests that robust PAC learning under KL
uncertainty may have sample-complexity behavior of a
different, possibly non-polynomial, nature in the accuracy
parameter.  Determining the correct KL robust PAC rates is
therefore a natural open problem.

\acks{The authors received no third-party funding in direct support
of this work and declare that they have no competing interests.}

\appendix

\section{Supplementary Material for Section~\ref{sec:EI}}\label{app:EI-proofs}

We first derive the closed form for \(k=2\), and then prove the
regularity and discrepancy bounds for general \(k>1\).

\subsection{Proof of (\ref{eq:q_2_closed_form})}\label{sec:eq:q_2_closed_form}

Fix \(p\in(0,1)\).  For \(k=2\), the Bernoulli divergence
\eqref{eq:D_k_q_p} becomes
\[
D_2(q\Vert p)
=
\frac12
\left(
\frac{q^2}{p}
+
\frac{(1-q)^2}{1-p}
-
1
\right).
\]
A direct simplification gives
\[
\frac{q^2}{p}
+
\frac{(1-q)^2}{1-p}
-
1
=
\frac{(q-p)^2}{p(1-p)}.
\]
Indeed, after putting the left-hand side over the common denominator
\(p(1-p)\), the numerator is
\[
q^2(1-p)+p(1-q)^2-p(1-p)
=
q^2-2pq+p^2
=
(q-p)^2 .
\]
Hence
\[
D_2(q\Vert p)
=
\frac{(q-p)^2}{2p(1-p)}.
\]

By the one-dimensional characterization \eqref{eq:1D-optimise}, we have
\[
\q{2,\rho}{p}
=
\sup
\left\{
q\in[p,1]:
\frac{(q-p)^2}{2p(1-p)}
\le \rho
\right\}.
\]
Since \(q\ge p\), the constraint is equivalent to
\[
q-p
\le
\sqrt{2\rho\,p(1-p)}.
\]
Therefore
\[
\q{2,\rho}{p}
=
\min\left\{
1,\,
p+\sqrt{2\rho\,p(1-p)}
\right\}.
\]

It remains only to identify the clipping point.  For \(p\in(0,1)\),
\[
p+\sqrt{2\rho\,p(1-p)}\ge1
\]
is equivalent to
\[
\sqrt{2\rho\,p(1-p)}\ge1-p.
\]
Squaring, which is legitimate since both sides are nonnegative, gives
\(
2\rho\,p(1-p)\ge(1-p)^2.
\)
Since \(p<1\), this is equivalent to
\(
2\rho p\ge1-p,
\)
or
\(
p\ge {1}/{(1+2\rho)}.
\)
Thus
\[
p_c(\rho;2)=\frac{1}{1+2\rho}.
\]
Consequently,
\[
\q{2,\rho}{p}
=
p+\sqrt{2\rho\,p(1-p)}
\qquad
\text{for }0<p<p_c(\rho;2),
\]
while
\[
\q{2,\rho}{p}=1
\qquad
\text{for }p\ge p_c(\rho;2).
\]

Finally, the endpoint \(p=0\) is covered by the convention
\(\q{2,\rho}{0}=0\), and the displayed formula also gives
\(0+\sqrt{2\rho\cdot0\cdot1}=0\).  The endpoint \(p=1\) lies in the
clipped branch and gives \(\q{2,\rho}{1}=1\).  This proves
\eqref{eq:q_2_closed_form}.
\hfill\BlackBox

\subsection{Proof of Lemma \ref{lem:q(p)_well_defined}}
\label{sec:EI_proofs}

For \(0<p<p_c(\rho;k)\), \eqref{eq:q(p)_def} and
\eqref{eq:D_k_q_p} give
\begin{align}
\frac{\bigl(\q{k,\rho}{p}\bigr)^k}{p^{k-1}}
+
\frac{\bigl(1-\q{k,\rho}{p}\bigr)^k}
     {(1-p)^{k-1}}
=
1+k(k-1)\rho.
\label{eq:q(p)_rho_equality}
\end{align}

\medskip
\noindent\emph{Differentiability and derivative bounds.}
We compute \(\frac{d\q{k,\rho}{p}}{dp}\) for arbitrary $0<p<p_c(\rho;k)$ using the {\sl implicit function theorem} \citep{rudin1976principles}.
Define
\[
F(p,q)
\coloneqq 
\frac{q^k}{p^{k-1}}
+
\frac{(1-q)^k}{(1-p)^{k-1}}
-
\bigl(1+k(k-1)\rho\bigr).
\]
Then,
\(
F(p,\q{k,\rho}{p})=0
\). Since \(\q{k,\rho}{p}\in(p,1)\), the function \(F\) is continuously
differentiable in a neighborhood of
\((p,\q{k,\rho}{p})\). Its partial derivatives are
\begin{align*}
\frac{\partial F}{\partial p}(p,q)
&=
(k-1)\left(
\frac{(1-q)^k}{(1-p)^k}
-
\frac{q^k}{p^k}
\right),
\\[5pt]
\frac{\partial F}{\partial q}(p,q)
&=
k\left(
\frac{q^{k-1}}{p^{k-1}}
-
\frac{(1-q)^{k-1}}{(1-p)^{k-1}}
\right).
\end{align*}
Introduce
\[
u\coloneqq \frac{\q{k,\rho}{p}}{p},
\qquad
v\coloneqq \frac{1-\q{k,\rho}{p}}{1-p}.
\]
Since \(\q{k,\rho}{p}\in(p,1)\), we have \(u>1\) and
\(0<v<1\). Hence
\[
\frac{\partial F}{\partial q}(p,\q{k,\rho}{p})
=
k(u^{k-1}-v^{k-1})
>
0.
\]
By the implicit function theorem,
\(p\mapsto\q{k,\rho}{p}\) is differentiable locally, and
\[
\frac{d\q{k,\rho}{p}}{dp}
=
-
\frac{
\frac{\partial F}{\partial p}(p,\q{k,\rho}{p})
}{
\frac{\partial F}{\partial q}(p,\q{k,\rho}{p})
}.
\]
Since
\[
\frac{\partial F}{\partial p}(p,\q{k,\rho}{p})
=
(k-1)(v^k-u^k),
\]
we obtain
\begin{align}
\frac{d\q{k,\rho}{p}}{dp}
=
\frac{k-1}{k} \cdot 
\frac{u^k-v^k}{u^{k-1}-v^{k-1}}.
\label{eq:dq/dp_u_v}
\end{align}

We first compare the derivative to the scale \(u\). Write
\(
\theta\coloneqq v/u
\).
Since \(u>1\) and \(0<v<1\), we have \(\theta\in(0,1)\). Owing to 
\eqref{eq:dq/dp_u_v},
\[
\frac{d\q{k,\rho}{p}}{dp}
=
u\cdot
\frac{k-1}{k}
\frac{1-\theta^k}{1-\theta^{k-1}}.
\]
For \(\theta\in(0,1)\),
\[
1
\le
\frac{1-\theta^k}{1-\theta^{k-1}}
\le
\frac{k}{k-1}.
\]
Indeed, the lower bound follows from
\(\theta^k\le \theta^{k-1}\). For the upper bound,
\[
1-\theta^k
=
k\int_\theta^1 t^{k-1}\,dt
\le
k\int_\theta^1 t^{k-2}\,dt
=
\frac{k}{k-1}(1-\theta^{k-1}),
\]
where the inequality uses \(t^{k-1}\le t^{k-2}\) on \((0,1]\).
Therefore
\begin{align}
\frac{k-1}{k}\,u
\le
\frac{d\q{k,\rho}{p}}{dp}
\le
u .
\label{eq:dqdp_comparable_to_u}
\end{align}

It remains to bound \(u\) above and below. The defining equation
\eqref{eq:q(p)_rho_equality} is equivalently
\begin{align}
p u^k+(1-p)v^k
=
1+k(k-1)\rho .
\label{eq:uv_constraint}
\end{align}

We first prove the lower bound for $u$. Since \(0<v<1\), we have \(v^k\le1\), and hence
\[
p u^k
=
1+k(k-1)\rho-(1-p)v^k
\ge
p+k(k-1)\rho .
\]
Thus
\[
u^k
\ge
1+\frac{k(k-1)\rho}{p}.
\]
Consequently,
\[
u
\ge
\left(
1+\frac{k(k-1)\rho}{p}
\right)^{1/k}.
\]
Using
\[
(1+s)^{1/k}
\ge
2^{1/k-1}\bigl(1+s^{1/k}\bigr),
\qquad s\ge0,
\]
we obtain
\[
u
\ge
2^{1/k-1}
+
2^{1/k-1}\bigl(k(k-1)\bigr)^{1/k}
\left(\frac{\rho}{p}\right)^{1/k}.
\]
Combining this with the left inequality in
\eqref{eq:dqdp_comparable_to_u}, we get
\[
\frac{d\q{k,\rho}{p}}{dp}
\ge
\tilde C_k
+
\tilde c_k
\left(\frac{\rho}{p}\right)^{1/k},
\]
where one may take
\[
\tilde C_k
\coloneqq 
\frac{k-1}{k}\,2^{1/k-1},
\qquad
\tilde c_k
\coloneqq 
\frac{k-1}{k}\,
2^{1/k-1}
\bigl(k(k-1)\bigr)^{1/k}.
\]

We now prove the upper bound. Set
\[
r(p)\coloneqq \q{k,\rho}{p}-p.
\]
Then \(0<r(p)<1-p\), and
\[
\q{k,\rho}{p}=p+r(p).
\]
Substituting this into \eqref{eq:q(p)_rho_equality} gives
\[
\frac{(p+r(p))^k}{p^{k-1}}
+
\frac{(1-p-r(p))^k}{(1-p)^{k-1}}
=
1+k(k-1)\rho .
\]
Subtracting
\[
\frac{p^k}{p^{k-1}}
+
\frac{(1-p)^k}{(1-p)^{k-1}}
=
1
\]
yields
\[
\frac{(p+r(p))^k-p^k}{p^{k-1}}
=
k(k-1)\rho
+
\frac{(1-p)^k-(1-p-r(p))^k}{(1-p)^{k-1}} .
\]
The map
\[
y\mapsto (y+r(p))^k-y^k
\]
is increasing on \([0,\infty)\), and therefore
\[
(p+r(p))^k-p^k
\ge
r(p)^k.
\]
Also, by the mean value theorem applied on
\([1-p-r(p),1-p]\),
\[
(1-p)^k-(1-p-r(p))^k
\le
k(1-p)^{k-1}r(p).
\]
Therefore
\[
\frac{r(p)^k}{p^{k-1}}
\le
k(k-1)\rho+k r(p).
\]
Now set
\(
x\coloneqq r(p)/p.
\)
Dividing the previous display by \(p\), we obtain
\[
x^k
\le
\frac{k(k-1)\rho}{p}
+
kx .
\]
Let
\[
\Lambda\coloneqq \frac{k(k-1)\rho}{p}.
\]
If \(x^k\le2\Lambda\), then
\(
x
\le
2^{1/k}\Lambda^{1/k}
\).
If \(x^k>2\Lambda\), then
\[
x^k
\le
\Lambda+kx
<
\frac{x^k}{2}+kx,
\]
and hence
\(
x^{k-1}<2k
\),
leading, in this case, to
\(
x<(2k)^{1/(k-1)}
\).
Combining the two cases,
\[
x
\le
(2k)^{1/(k-1)}
+
2^{1/k}\bigl(k(k-1)\bigr)^{1/k}
\left(\frac{\rho}{p}\right)^{1/k}.
\]
Since \(u=1+x\), it follows that
\[
u
\le
C_k
+
c_k
\left(\frac{\rho}{p}\right)^{1/k},
\]
where one may take
\[
C_k
\coloneqq 
1+(2k)^{1/(k-1)},
\qquad
c_k
\coloneqq 
2^{1/k}\bigl(k(k-1)\bigr)^{1/k}.
\]
Combining this with the right inequality in
\eqref{eq:dqdp_comparable_to_u}, we get
\[
\frac{d\q{k,\rho}{p}}{dp}
\le
C_k
+
c_k
\left(\frac{\rho}{p}\right)^{1/k}.
\]

The lower and upper derivative bounds prove \eqref{eq:q'_bound}.

\medskip
\noindent\emph{Continuity and absolute continuity of the extended map.}
The implicit-function argument above gives \(C^1\) regularity of
\(p\mapsto\q{k,\rho}{p}\) on \((0,p_c(\rho;k))\). Since the extension is
constant on \([p_c(\rho;k),1]\), it remains to check the behavior at the two
endpoints of the interior branch.

As \(p\downarrow0\), the defining equation gives
\[
\frac{\q{k,\rho}{p}^k}{p^{k-1}}
\le
1+k(k-1)\rho,
\]
and hence
\[
\q{k,\rho}{p}
\le
\bigl(1+k(k-1)\rho\bigr)^{1/k}
p^{1-\frac1k}
\to0.
\]
Thus
\[
\q{k,\rho}{p}\to0=\q{k,\rho}{0}.
\]

It remains to consider \(p\uparrow p_c(\rho;k)\). Since the derivative
computed above is positive, the interior branch is increasing.
Thus the limit
\[
\ell\coloneqq \lim_{p\uparrow p_c(\rho;k)}\q{k,\rho}{p}
\]
exists and belongs to \([p_c(\rho;k),1]\). Passing to the limit in the
defining equation gives
\[
D_k(\ell\Vert p_c(\rho;k))=\rho.
\]
By the definition of \(p_c(\rho;k)\),
\[
D_k(1\Vert p_c(\rho;k))=\rho.
\]
Moreover, \(q\mapsto D_k(q\Vert p_c(\rho;k))\) is strictly increasing
on \((p_c(\rho;k),1]\). Therefore \(\ell=1\), and hence
\[
\q{k,\rho}{p}\to1=\q{k,\rho}{p_c(\rho;k)}
\qquad\text{as }p\uparrow p_c(\rho;k).
\]
This proves continuity of the extended map on \([0,1]\).

We now prove absolute continuity. The derivative upper bound gives,
for \(0<t<p_c(\rho;k)\),
\[
0
\le
\frac{d\q{k,\rho}{t}}{dt}
\le
C_k+c_k\left(\frac{\rho}{t}\right)^{1/k}.
\]
The right-hand side is integrable on \((0,p_c(\rho;k))\), since \(k>1\).
Therefore, for every \(0<a<b<p_c(\rho;k)\),
\[
\q{k,\rho}{b}-\q{k,\rho}{a}
=
\int_a^b
\frac{d\q{k,\rho}{t}}{dt}\,dt
\le
\int_a^b
\left[
C_k+c_k\left(\frac{\rho}{t}\right)^{1/k}
\right]dt .
\]
By the endpoint continuity just proved, the same bound extends to
all \(0\le a<b\le p_c(\rho;k)\) by passing to the boundary. Consequently,
for every finite collection of disjoint intervals
\(\{[a_i,b_i]\}_i\subset[0,p_c(\rho;k)]\),
\[
\sum_i
\left|
\q{k,\rho}{b_i}
-
\q{k,\rho}{a_i}
\right|
\le
\int_{\cup_i [a_i,b_i]}
\left[
C_k+c_k\left(\frac{\rho}{t}\right)^{1/k}
\right]dt .
\]
Since the dominating function is integrable, the right-hand side
can be made arbitrarily small whenever the total length of
\(\cup_i [a_i,b_i]\) is sufficiently small. Hence
\(p\mapsto\q{k,\rho}{p}\) is absolutely continuous on
\([0,p_c(\rho;k)]\). Since it is constant on \([p_c(\rho;k),1]\), the extended map is
absolutely continuous on \([0,1]\).
\hfill\BlackBox

\subsection{Proof of Theorem \ref{thm:qp-qp'_bounds}}\label{app:thm:qp-qp'_bounds}

Throughout the proof, \(\rho>0\). By
Lemma~\ref{lem:q(p)_well_defined}, the map
\(p\mapsto \q{k,\rho}{p}\) is continuous on \([0,1]\),
differentiable on \((0,p_c(\rho;k))\), and satisfies
\begin{equation}\label{eq:derivative-bounds}
\tilde C_k
+
\tilde c_k\left(\frac{\rho}{t}\right)^{1/k}
\le
\frac{d\q{k,\rho}{t}}{dt}
\le
C_k
+
c_k\left(\frac{\rho}{t}\right)^{1/k},
\end{equation}
whenever \(0<t<p_c(\rho;k)\). 
Moreover,
\[
\q{k,\rho}{p}=1
\qquad\text{for all }p\in[p_c(\rho;k),1].
\]
It follows, in particular, that \(p\mapsto \q{k,\rho}{p}\)
is nondecreasing on \([0,1]\).

\bigskip
\noindent
\emph{Proof of~\eqref{eq:disc-upper}.} By the definition of clipping,
\[
\q{k,\rho}{p}
=
\q{k,\rho}{\bar p},
\qquad
\q{k,\rho}{p'}
=
\q{k,\rho}{\bar p^{\,\prime}}.
\]
Hence it suffices to control the increment between the
clipped endpoints.

By symmetry, assume
\(
\bar p\le \bar p^{\,\prime}
\).
If \(\cd=0\), then the increment is zero and there is
nothing to prove. Otherwise,
\(
0\le \bar p<\bar p^{\,\prime}\le p_c(\rho;k).
\)
Since \(p\mapsto \q{k,\rho}{p}\) is absolutely continuous
on \([0,1]\), the fundamental theorem of calculus asserts that
\[
\q{k,\rho}{p'}
-
\q{k,\rho}{p}
=
\q{k,\rho}{\bar p^{\,\prime}}
-
\q{k,\rho}{\bar p}
=
\int_{\bar p}^{\bar p^{\,\prime}}
\frac{d\q{k,\rho}{t}}{dt}\,dt.
\]

Using the derivative upper bound seen at~\eqref{eq:derivative-bounds}, 
we may write
\[
\q{k,\rho}{p'}
-
\q{k,\rho}{p}
\le
C_k\cd
+
c_k\rho^{1/k}
\int_{\bar p}^{\bar p^{\,\prime}} t^{-1/k}\,dt .
\]
Since \(k>1\),
\[
\int_{\bar p}^{\bar p^{\,\prime}} t^{-1/k}\,dt
=
\frac{k}{k-1}
\left[
(\bar p^{\,\prime})^{1-\frac1k}
-
\bar p^{\,1-\frac1k}
\right].
\]
Furthermore,
\[
(\bar p^{\,\prime})^{1-\frac1k}
-
\bar p^{\,1-\frac1k}
=
(\bar p^{\,\prime})^{1-\frac1k}
\left[
1-
\left(\frac{\bar p}{\bar p^{\,\prime}}\right)^{1-\frac1k}
\right].
\]
Because \(1-\frac1k\in(0,1)\), we have
\[
\left(\frac{\bar p}{\bar p^{\,\prime}}\right)^{1-\frac1k}
\ge
\frac{\bar p}{\bar p^{\,\prime}},
\]
and hence
\[
1-
\left(\frac{\bar p}{\bar p^{\,\prime}}\right)^{1-\frac1k}
\le
1-\frac{\bar p}{\bar p^{\,\prime}}.
\]
Consequently,
\[
(\bar p^{\,\prime})^{1-\frac1k}
-
\bar p^{\,1-\frac1k}
\le
(\bar p^{\,\prime})^{1-\frac1k}
\left(
1-\frac{\bar p}{\bar p^{\,\prime}}
\right)
=
\frac{\bar p^{\,\prime}-\bar p}
     {(\bar p^{\,\prime})^{1/k}}
=
\frac{\cd}
     {(\bar p^{\,\prime})^{1/k}}.
\]
Therefore,
\[
\q{k,\rho}{p'}
-
\q{k,\rho}{p}
\le
C_k\cd
+
\frac{k}{k-1}c_k\rho^{1/k}
\frac{\cd}
     {(\bar p^{\,\prime})^{1/k}}.
\]
Under the standing assumption
\(\bar p\le\bar p^{\,\prime}\), we have
\(
\bar p\vee\bar p^{\,\prime}
=
\bar p^{\,\prime}
\).
Thus,
\[
\q{k,\rho}{p'}
-
\q{k,\rho}{p}
\le
C_k\cd
+
\frac{k}{k-1}c_k\rho^{1/k}
\frac{\cd}
     {(\bar p\vee\bar p^{\,\prime})^{1/k}}.
\]
The case \(\bar p^{\,\prime}\le\bar p\) follows by exchanging
the roles of \(p\) and \(p'\). This proves
\[
\bigl|
\q{k,\rho}{p}
-
\q{k,\rho}{p'}
\bigr|
\le
C_k\cd
+
\frac{k}{k-1}c_k\rho^{1/k}
\frac{\cd}
     {(\bar p\vee\bar p^{\,\prime})^{1/k}},
\]
with the convention that the second term is zero when
\(\cd=0\).
Absorbing the two \(k\)-dependent coefficients into a single constant
proves~\eqref{eq:disc-upper}.

\bigskip
\noindent
\emph{Proof of~\eqref{eq:disc-lower}.} Assume
\[
0<p<p'<p_c(\rho;k),
\qquad
\Delta=p'-p.
\]
Then
\[
[p,p']
\subset
(0,p_c(\rho;k)),
\]
where \(t\mapsto\q{k,\rho}{t}\) is differentiable.  Hence, by the
absolute continuity of the event-inflation map,
\[
\q{k,\rho}{p'}-\q{k,\rho}{p}
=
\int_p^{p'}
\frac{d\q{k,\rho}{t}}{dt}\,dt.
\]
Using the derivative lower bound in~\eqref{eq:derivative-bounds},
\[
\q{k,\rho}{p'}-\q{k,\rho}{p}
\ge
\tilde C_k\Delta
+
\tilde c_k\rho^{1/k}
\int_p^{p'} t^{-1/k}\,dt.
\]
Since \(t\le p'\) throughout the interval \([p,p']\),
\[
t^{-1/k}\ge (p')^{-1/k}.
\]
Therefore,
\[
\int_p^{p'} t^{-1/k}\,dt
\ge
\frac{p'-p}{(p')^{1/k}}
=
\frac{\Delta}{(p')^{1/k}}.
\]
Consequently,
\[
\q{k,\rho}{p'}-\q{k,\rho}{p}
\ge
\tilde C_k\Delta
+
\tilde c_k\rho^{1/k}
\frac{\Delta}{(p')^{1/k}}.
\]
After decreasing \(\tilde c_k\), if necessary, this is precisely
\eqref{eq:disc-lower}.

It remains to treat the boundary cases allowed in the statement of the
theorem.  If \(p=0\) and \(p'<p_c(\rho;k)\), apply the preceding
inequality with \(p>0\) and let \(p\downarrow0\).  By the continuity of
the event-inflation map at zero and the identity
\(\q{k,\rho}{0}=0\), this gives
\[
\q{k,\rho}{p'}
-
\q{k,\rho}{0}
\ge
\tilde c_k\,p'
\left[
1+
\left(\frac{\rho}{p'}\right)^{1/k}
\right].
\]

If \(p'=p_c(\rho;k)\), apply the preceding inequality with
\(p'<p_c(\rho;k)\) and let \(p'\uparrow p_c(\rho;k)\).  Continuity of
the event-inflation map at \(p_c(\rho;k)\) yields
\[
\q{k,\rho}{p_c(\rho;k)}
-
\q{k,\rho}{p}
\ge
\tilde c_k\,
\bigl(p_c(\rho;k)-p\bigr)
\left[
1+
\left(
\frac{\rho}{p_c(\rho;k)}
\right)^{1/k}
\right].
\]
The same argument covers the case \(p=0\) and
\(p'=p_c(\rho;k)\).  This proves~\eqref{eq:disc-lower} throughout the
range
\[
0\le p<p'\le p_c(\rho;k).
\]
\hfill\BlackBox

\section{Supplementary Material for Section~\ref{sec:PAC_rates}}\label{app:PAC_rates}

We first establish the VC bounds, then prove the realizable and
agnostic sample-complexity results, and finally verify the
self-consistency condition used in the upper bounds.

\subsection{Scale-Sensitive VC Bounds and Their Clipped Versions}

Before proving Lemma~\ref{lem:clipped-vc-master}, we first establish the following unclipped version.
\begin{lemma}[Scale-sensitive VC bound]
\label{lem:vc-master}
Let \(\hypoclass\) be a hypothesis class with
\(\VC(\hypoclass)=d\ge1\), let \(n\ge d\), and let
\(\delta\in(0,1)\).
There exists a universal constant $C>0$ such that with probability at
least $1-\delta$ over the choice of the sample, 
both
\begin{align}
p_h-\hat p_h
&\le
C\sqrt{\frac{p_h\GammaMac}{n}}
\label{eq:vc-master}
\qquad\text{and}\qquad
\hat p_h-p_h
\le
C\sqrt{\frac{\hat p_h\GammaMac}{n}}
\end{align}
hold uniformly over all $h\in\hypoclass$.
In particular, both
\begin{align}
\label{eq:master_vc_abs}
\Delta_h
&\le
C\left(
\sqrt{\frac{p_h\GammaMac}{n}}
+
\frac{\GammaMac}{n}
\right)
\end{align}
and
\begin{align}
\label{eq:master_vc_real}
\hat p_h=0
\quad\Longrightarrow\quad p_h\le C\frac{\GammaMac}{n}  \end{align}
hold uniformly over all $h\in\hypoclass$.
\end{lemma}

\subsubsection{Proof of Lemma~\ref{lem:vc-master}: Scale-Sensitive VC Bound}\label{app:vc-master}

For $h\in\hypoclass$, define
\(
f_h(x,y)\coloneqq \ind{h(x)\neq y}
\)
and let
\(
\mathcal F
\coloneqq 
\bigl\{f_h
: h\in\hypoclass\bigr\}
\).
Then
\[
P f_h
\coloneqq
\EE_{(X,Y)\sim P}\!\left[f_h(X,Y)\right]
=
p_h,
\qquad
P_n f_h
\coloneqq
\EE_{(X,Y)\sim P_n}\!\left[f_h(X,Y)\right]
=
\hat p_h.
\]
Applying the two one-sided inequalities in
\citet[Theorem~5.1]{boucheron2005theory} with confidence parameter
\(\delta/2\), and then taking a union bound, shows that with
probability at least \(1-\delta\), simultaneously for every
\(f\in\mathcal F\),
\begin{align}
P f-P_n f
&\le
2\sqrt{
P f\,
\frac{
\log S_{\mathcal F}(X_1^{2n})
+
\log(8/\delta)
}{n}
},
\label{eq:bbl-upper}
\\[6pt]
P_n f-P f
&\le
2\sqrt{
P_n f\,
\frac{
\log S_{\mathcal F}(X_1^{2n})
+
\log(8/\delta)
}{n}
}.
\label{eq:bbl-lower}
\end{align}
Here \(S_{\mathcal F}(X_1^{2n})\) denotes the shattering coefficient
of \(\mathcal F\) on an i.i.d.\ sample of size \(2n\).

For any fixed labeled sample, the error patterns generated by
\[
\mathcal F
=
\left\{
(x,y)\mapsto\ind{h(x)\neq y}:
h\in\hypoclass
\right\}
\]
are obtained from the prediction patterns generated by
\(\hypoclass\) by coordinatewise relabeling according to the fixed
labels. Hence
\(\VC(\mathcal F)\leq\VC(\hypoclass)=d\).

Since \(\VC(\mathcal F)\le d\) and \(n\ge d\), the Sauer--Shelah
lemma gives
\[
S_{\mathcal F}(X_1^{2n})
\le
\left(
\frac{2\e n}{d}
\right)^d
\]
almost surely.  Therefore,
\[
\log S_{\mathcal F}(X_1^{2n})
+
\log(8/\delta)
\le
d\log\left(
\frac{2\e n}{d}
\right)
+
\log\left(
\frac8\delta
\right)
=
\Gamma_{d,n}(\delta).
\]
Applying \eqref{eq:bbl-upper} and \eqref{eq:bbl-lower} to
\(f_h\) gives, simultaneously for every \(h\in\hypoclass\),
\[
p_h-\hat p_h
\le
C\sqrt{
\frac{p_h\Gamma_{d,n}(\delta)}{n}
}
\]
and
\[
\hat p_h-p_h
\le
C\sqrt{
\frac{\hat p_h\Gamma_{d,n}(\delta)}{n}
}.
\]
This proves the first part of the lemma.

\medskip
To establish \eqref{eq:master_vc_abs}, set
\[
a
\coloneqq
\frac{\Gamma_{d,n}(\delta)}{n}
\]
and fix \(h\in\hypoclass\).
If \(\hat p_h\leq p_h\), then
\[
\Delta_h
\leq
C\sqrt{p_h a}.
\]
If \(\hat p_h>p_h\), then
\[
\Delta_h
\leq
C\sqrt{\hat p_h a}
=
C\sqrt{(p_h+\Delta_h)a}
\leq
C\sqrt{p_h a}
+
C\sqrt{\Delta_h a}.
\]
Using
\[
C\sqrt{\Delta_h a}
\leq
\frac{\Delta_h}{2}
+
\frac{C^2a}{2}
\]
and rearranging gives
\[
\Delta_h
\leq
C\left(
\sqrt{p_h a}+a
\right)
\]
after increasing the universal constant. This proves
\eqref{eq:master_vc_abs} uniformly over \(h\in\hypoclass\).

\medskip
Finally, if \(\hat p_h=0\), then
\eqref{eq:vc-master} gives
\[
p_h
\leq
C\sqrt{
\frac{p_h\Gamma_{d,n}(\delta)}{n}
}.
\]
Thus either \(p_h=0\), or division by \(\sqrt{p_h}\) and squaring
give
\[
p_h
\leq
C^2\frac{\Gamma_{d,n}(\delta)}{n}.
\]
This proves \eqref{eq:master_vc_real} after enlarging the universal
constant.
\hfill\BlackBox

\subsubsection{Proof of Lemma~\ref{lem:clipped-vc-master}: Clipped Scale-Sensitive VC Bound}
\label{app:clipped-vc-master}

Assume the event on which
\eqref{eq:master_vc_abs} and
\eqref{eq:master_vc_real} hold uniformly, and set
\[
a
\coloneqq
\frac{\GammaMac}{n}.
\]
Fix \(h\in\hypoclass\). Since
\(p\mapsto p\wedge p_c(\rho;k)\) is \(1\)-Lipschitz,
\[
\cd_h
\leq
\Delta_h.
\]

If \(\cd_h=0\), there is nothing to prove. Otherwise,
\[
p_h
\leq
\cp_h+\Delta_h.
\]
Indeed, this is immediate when
\(p_h\leq p_c(\rho;k)\); if
\(p_h>p_c(\rho;k)\), then
\(\cd_h>0\) implies
\(\hat p_h<p_c(\rho;k)\), and hence
\[
p_h-p_c(\rho;k)
\leq
p_h-\hat p_h
=
\Delta_h.
\]

Therefore, by \eqref{eq:master_vc_abs},
\begin{align*}
\Delta_h
\leq
C\left(
\sqrt{p_h a}+a
\right)
\leq
C\left(
\sqrt{\cp_h a}
+
\sqrt{\Delta_h a}
+
a
\right).
\end{align*}
Using
\[
C\sqrt{\Delta_h a}
\leq
\frac{\Delta_h}{2}
+
\frac{C^2a}{2}
\]
and rearranging gives
\[
\Delta_h
\leq
C\left(
\sqrt{\cp_h a}+a
\right)
\]
after increasing the universal constant. Since
\(\cd_h\leq\Delta_h\), this proves
\eqref{eq:master_vc_abs_clip}.

Finally, if \(\cep_h=0\), then
\(\hat p_h=0\), because \(p_c(\rho;k)>0\).
Equation~\eqref{eq:master_vc_real} therefore gives
\[
p_h
\leq
C\frac{\GammaMac}{n},
\]
and hence
\[
\cp_h
\leq
p_h
\leq
C\frac{\GammaMac}{n}.
\]
This proves \eqref{eq:master_vc_real_clip}.
\hfill\BlackBox

\subsection{Sample Complexity of Realizably Robust Learning: The Proofs}\label{app:real-rates-bounds}

We first prove the upper bound and then the lower bound.

\bigskip

\subsubsection{Proof of Theorem~\ref{thm:main_upper_realizable}:
Realizable Upper Bounds}

Let \(h^\star\in\hypoclass\) satisfy
\(\err{h^\star}{P}=0\). Then
\(\hat p_{h^\star}=0\) almost surely, and hence
\[
\risk{k,\rho}{h^\star}{\Ped}
=
\q{k,\rho}{0}
=
0.
\]
Since \(\erm\) is a DRERM and robust risks are nonnegative,
\[
\risk{k,\rho}{\erm}{\Ped}
=
0.
\]
Moreover, since the empirical distribution \(\Ped\) is feasible in
the supremum defining the empirical robust risk,
\[
\risk{k,\rho}{h}{\Ped}
\geq
\err{h}{\Ped}
=
\hat p_h
\]
for every \(h\in\hypoclass\). Therefore,
\(\risk{k,\rho}{\erm}{\Ped}=0\) implies
\[
\hat p_{\erm}=0,
\qquad
\cep_{\erm}=0,
\]
where for \(p\in[0,1]\), we recall that
\[
\bar p\coloneqq p\wedge p_c(\rho;k).
\]

The sample-size condition in the theorem statement implies \(n\geq d\), after
increasing \(C_k\) if necessary. Hence
Lemma~\ref{lem:clipped-vc-master} applies, and with probability at
least \(1-\delta\), \eqref{eq:master_vc_real_clip} gives
\[
\cp_{\erm}
\leq
C\frac{\Gamma_{d,n}(\delta)}{n}.
\]
On this event, the derivation from
\eqref{eq:real-infaltion} through
\eqref{eq:samp_suff_cond_real} shows that
\[
\risk{k,\rho}{\erm}{P}
\leq
\eps
\]
whenever
\[
n
\geq
C'_k\,
\Gamma_{d,n}(\delta)\,
\Psi_{k,\rho}^{\mathrm{real}}(\eps)
\]
for a suitable \(C'_k\geq1\).

Applying Lemma~\ref{lem:self-consistency} with
\[
\Psi
=
C'_k\Psi_{k,\rho}^{\mathrm{real}}(\eps),
\]
the theorem's sample-size condition implies this implicit condition
after increasing \(C_k\) if necessary.
This proves the DRERM assertion.

\medskip
\noindent\emph{Log-free sample-complexity upper bound.}
By \citet[Theorem~2]{hanneke2016optimal}, together with the elementary
cases \(|\hypoclass|\leq2\), there exists a universal constant
\(C_0>0\) such that, for every \(\alpha,\delta\in(0,1)\), there is a
learning rule
\[
r_n^{\mathrm{Han}}:(\X\times\Y)^n\to\mathcal G,
\]
not required to be proper with respect
to \(\hypoclass\), satisfying
\[
\Pr_{S_n\sim P^n}
\left[
\err{r_n^{\mathrm{Han}}(S_n)}{P}
\leq\alpha
\right]
\geq1-\delta
\]
for every distribution \(P\) realizable with respect to
\(\hypoclass\), whenever
\[
n
\geq
\frac{C_0}{\alpha}
\left(
d+\log(1/\delta)
\right).
\]
The rule \(r_n^{\mathrm{Han}}\) is a finite majority vote of members
of \(\hypoclass\). Since these classifiers are measurable,
\(r_n^{\mathrm{Han}}(S_n)\) is measurable for every \(S_n\), and
hence belongs to \(\mathcal G\).

Equation~\eqref{eq:gen-k-regime-capture}, together with
\(\q{k,0}{p}=p\), implies that there exists a constant \(B_k\geq1\)
such that, for every \(\rho\geq0\) and \(p\in[0,1]\),
\[
\q{k,\rho}{p}
\leq
B_k
\left(
\bar p+\rho^{1/k}\bar p^{1-1/k}
\right)
\leq
B_k
\left(
p+\rho^{1/k}p^{1-1/k}
\right).
\]
The second inequality follows from \(\bar p\leq p\) and
\(1-1/k>0\).
Choose \(\beta_k\in(0,1]\) sufficiently small that
\[
B_k
\left(
\beta_k+\beta_k^{1-1/k}
\right)
\leq1,
\]
and set
\[
\alpha
\coloneqq
\frac{\beta_k}{
\Psi_{k,\rho}^{\mathrm{real}}(\eps)
}.
\]
Since \(\eps\in(0,1)\), we have \(\alpha\in(0,1)\). Moreover, the
definition of \(\Psi_{k,\rho}^{\mathrm{real}}(\eps)\) gives
\[
\frac{1}{
\Psi_{k,\rho}^{\mathrm{real}}(\eps)
}
\leq
\eps
\]
and
\[
\rho^{1/k}
\left(
\frac{1}{
\Psi_{k,\rho}^{\mathrm{real}}(\eps)
}
\right)^{1-1/k}
\leq
\eps.
\]

Let
\[
p_n
\coloneqq
\err{r_n^{\mathrm{Han}}(S_n)}{P},
\qquad
\bar p_n
\coloneqq
p_n\wedge p_c(\rho;k).
\]
On the event \(p_n\leq\alpha\), the event-inflation identity applied
to the error event of \(r_n^{\mathrm{Han}}(S_n)\) gives
\begin{align*}
\risk{k,\rho}{r_n^{\mathrm{Han}}(S_n)}{P}
&=
\q{k,\rho}{p_n}
\\
&\leq
B_k
\left(
\bar p_n+\rho^{1/k}\bar p_n^{1-1/k}
\right)
\\
&\leq
B_k
\left(
p_n+\rho^{1/k}p_n^{1-1/k}
\right)
\\
&\leq
B_k
\left(
\alpha+\rho^{1/k}\alpha^{1-1/k}
\right)
\\
&\leq
B_k
\left(
\beta_k+\beta_k^{1-1/k}
\right)\eps
\\
&\leq
\eps.
\end{align*}
Consequently,
\[
\Pr_{S_n\sim P^n}
\left[
\risk{k,\rho}{r_n^{\mathrm{Han}}(S_n)}{P}
\leq\eps
\right]
\geq1-\delta
\]
whenever
\[
n
\geq
C_0\beta_k^{-1}
\Psi_{k,\rho}^{\mathrm{real}}(\eps)
\left(
d+\log(1/\delta)
\right).
\]
Since \(C_0\beta_k^{-1}\) depends only on \(k\), this proves the
claimed log-free sample-complexity upper bound and completes the
proof.
\hfill\BlackBox

\subsubsection{Proof of Theorem~\ref{thm:main_lower_realizable}: Realizable Lower Bounds}

Under the assumption \(|\hypoclass|\geq3\), the classical realizable
VC lower bound
\citep[Eq.~(1)]{hanneke2016optimal} provides universal constants
\(c_{\mathrm{vc}},\eps_{\mathrm{vc}},\delta_{\mathrm{vc}}>0\)
such that every learning rule satisfying
\[
\Pr_{S_n\sim P^n}
\left[
\err{r_n(S_n)}{P}\leq\alpha
\right]
\geq1-\delta
\]
for every distribution \(P\) realizable with respect to
\(\hypoclass\) must satisfy
\[
n
\geq
c_{\mathrm{vc}}
\frac{d+\log(1/\delta)}{\alpha},
\]
whenever
\(\alpha\leq\eps_{\mathrm{vc}}\) and
\(\delta\leq\delta_{\mathrm{vc}}\).

We shall also use the proper-learning lower bound in the final
``Furthermore'' clause of
\citet[Theorem~11]{bousquet2020proper}.
After adjusting its universal
constant, 
it asserts that, for every integer \(d\geq1\), there exist a
measurable instance space \(\X_d\) and a binary hypothesis class
\(\hypoclass_d\) on \(\X_d\), with \(\VC(\hypoclass_d)=d\), such that
every proper
learning rule satisfying the preceding ordinary realizable guarantee
with \(\hypoclass=\hypoclass_d\) must satisfy
\[
n
\geq
\frac{c_{\mathrm{prop}}}{\alpha}
\left[
d\log\left(\frac{\e}{\alpha}\right)
+
\log(1/\delta)
\right]
\]
for all
\[
\alpha\in(0,1/8),
\qquad
\delta\in(0,1/100),
\]
where \(c_{\mathrm{prop}}>0\) is universal.

Set
\[
\eps_0
\coloneqq
\eps_{\mathrm{vc}}\wedge\frac1{16},
\qquad
\delta_0
\coloneqq
\delta_{\mathrm{vc}}\wedge\frac1{200}.
\]
Fix \(k>1\), \(\rho\geq0\),
\(\eps\in(0,\eps_0]\), and
\(\delta\in(0,\delta_0]\).

For \(\rho=0\), define
\[
\alpha_{k,0}(\eps)\coloneqq\eps.
\]
For \(\rho>0\), continuity and strict monotonicity of
\(p\mapsto\q{k,\rho}{p}\) on
\([0,p_c(\rho;k)]\), together with
\[
\q{k,\rho}{0}=0,
\qquad
\q{k,\rho}{p_c(\rho;k)}=1,
\]
show that there is a unique
\[
\alpha_{k,\rho}(\eps)\in(0,p_c(\rho;k))
\]
such that
\begin{equation}
\label{eq:real-inverse-accuracy}
\q{k,\rho}{\alpha_{k,\rho}(\eps)}
=
\eps.
\end{equation}
Write
\[
\alpha\coloneqq\alpha_{k,\rho}(\eps).
\]

We first relate this ordinary accuracy to the realizable complexity
proxy. For \(\rho>0\), the lower bound
\eqref{eq:disc-lower}, applied with \(p=0\) and \(p'=\alpha\), and
the upper bound \eqref{eq:gen-k-regime-capture} give
\[
\eps
=
\q{k,\rho}{\alpha}
\asymp_k
\alpha+\rho^{1/k}\alpha^{1-1/k}.
\]
Inverting this relation yields
\begin{equation}
\label{eq:real-inverse-proxy}
\frac1{\alpha}
\asymp_k
\frac1{\eps}
+
\frac{\rho^{1/(k-1)}}{\eps^{k_\star}}
=
\Psi_{k,\rho}^{\mathrm{real}}(\eps).
\end{equation}
Indeed,
\[
\alpha+\rho^{1/k}\alpha^{1-1/k}
\lesssim_k\eps
\]
bounds each summand separately and hence gives the lower bound on
\(1/\alpha\) in \eqref{eq:real-inverse-proxy}. Conversely,
\[
\eps
\lesssim_k
\alpha+\rho^{1/k}\alpha^{1-1/k}
\]
implies that at least one of the two summands is bounded below by a
\(k\)-dependent constant multiple of \(\eps\), which gives the
corresponding upper bound on \(1/\alpha\). The case \(\rho=0\)
follows directly from \(\alpha=\eps\).

Now suppose that \(r_n\) satisfies the robust realizable guarantee in
the theorem. By the one-dimensional event-inflation identity,
strict monotonicity below \(p_c(\rho;k)\), and
\eqref{eq:real-inverse-accuracy},
\[
\left\{
\risk{k,\rho}{r_n(S_n)}{P}\leq\eps
\right\}
=
\left\{
\err{r_n(S_n)}{P}\leq\alpha
\right\}.
\]
Moreover, \(\q{k,\rho}{p}\geq p\) implies
\(\alpha\leq\eps\leq\eps_{\mathrm{vc}}\). Thus \(r_n\) satisfies the
ordinary realizable PAC guarantee at accuracy \(\alpha\), and the
classical lower bound gives
\[
n
\geq
c_{\mathrm{vc}}
\frac{d+\log(1/\delta)}{\alpha}.
\]
Combining this with \eqref{eq:real-inverse-proxy} proves
\[
n
\geq
\tilde c_k\,
\Psi_{k,\rho}^{\mathrm{real}}(\eps)
\bigl(d+\log(1/\delta)\bigr).
\]

Finally, take the pair \((\X_d,\hypoclass_d)\) supplied by
\citet[Theorem~11]{bousquet2020proper}. If \(r_n\) is proper, then
the same learning rule, viewed through the preceding reduction, is
an ordinary proper learner at accuracy \(\alpha\). The proper-learning
lower bound therefore gives
\[
n
\geq
\frac{c_{\mathrm{prop}}}{\alpha}
\left[
d\log\left(\frac{\e}{\alpha}\right)
+
\log(1/\delta)
\right].
\]
Since \eqref{eq:real-inverse-proxy} implies
\[
\frac1\alpha
\asymp_k
\Psi_{k,\rho}^{\mathrm{real}}(\eps),
\]
and both quantities are at least one, their logarithms are comparable
up to \(k\)-dependent constants. Consequently,
\[
n
\geq
\tilde c_k\,
\Psi_{k,\rho}^{\mathrm{real}}(\eps)
\left[
d\log\left(
\e\,\Psi_{k,\rho}^{\mathrm{real}}(\eps)
\right)
+
\log(1/\delta)
\right].
\]
This proves the proper-learning assertion and completes the proof.
\hfill\BlackBox

\subsection{Sample Complexity of Agnostically Robust Learning: The Proofs}

We first prove the upper bound and then the lower bound.

\subsubsection{
Proof of Theorem~\ref{thm:main_upper}:
Agnostic Upper Bound Achievable by DRERM
}

Let \(\erm=\erm(S_n)\) be a DRERM and set
\[
Z_n
\coloneqq
\sup_{h\in\hypoclass}
\left|
\risk{k,\rho}{h}{P}
-
\risk{k,\rho}{h}{\Ped}
\right|.
\]
By empirical robust-risk optimality,
\[
\risk{k,\rho}{\erm}{P}
-
\inf_{h\in\hypoclass}
\risk{k,\rho}{h}{P}
\leq
2Z_n.
\]

The sample-size condition
\eqref{eq:agn-rate-unabridged} implies \(n\geq d\), after increasing
\(C_k\) if necessary. Hence
Lemma~\ref{lem:clipped-vc-master} applies, and with probability at
least \(1-\delta\), the bound
\eqref{eq:master_vc_abs_clip} holds uniformly over
\(h\in\hypoclass\).

On this event, applying the derivation from
\eqref{eq:rob-risk-regime-capture-agn} through
\eqref{eq:suff_cond_agn} with target accuracy \(\eps/2\) gives
\[
Z_n
\leq
\frac{\eps}{2}
\]
whenever
\[
n
\geq
C'_k\,
\Gamma_{d,n}(\delta)\,
\Psi_{k,\rho}^{\agn}(\eps/2).
\]
Since
\[
\Psi_{k,\rho}^{\agn}(\eps/2)
\leq
2^{k_\star\vee2}
\Psi_{k,\rho}^{\agn}(\eps),
\]
it is enough that
\[
n
\geq
C''_k\,
\Gamma_{d,n}(\delta)\,
\Psi_{k,\rho}^{\agn}(\eps)
\]
for a suitable \(C''_k\geq1\).

Applying Lemma~\ref{lem:self-consistency} with
\[
\Psi
=
C''_k\Psi_{k,\rho}^{\agn}(\eps),
\]
the sample-size condition
\eqref{eq:agn-rate-unabridged} implies this implicit condition after
increasing \(C_k\) if necessary. Hence
\(Z_n\leq\eps/2\), and the claimed robust excess-risk bound follows.
\hfill\BlackBox

\subsubsection{Proof of Theorem~\ref{thm:main_lower_agnostic}: Agnostic Lower Bound}
\label{sec:thm:main_lower_agnostic}

We first prove the lower bound under the stronger failure guarantee
\[
\Pr_{S_n\sim P^n}
\left[
\risk{k,\rho}{r_n(S_n)}{P}
-
\inf_{h\in\hypoclass}\risk{k,\rho}{h}{P}
\geq\eps
\right]
\leq\delta
\]
for every distribution \(P\). The guarantee in the theorem gives the
corresponding bound only with failure threshold \(>\eps\).
Nevertheless, it implies the stronger-form guarantee at accuracy
\(2\eps\), since robust excess risk at least \(2\eps\) implies robust
excess risk greater than \(\eps\). Moreover,
\[
\Psi_{k,\rho}^{\agn}(2\eps)
\geq
2^{-(k_\star\vee2)}
\Psi_{k,\rho}^{\agn}(\eps).
\]
Thus, after decreasing \(\tilde\eps_k\) if necessary and adjusting
the \(k\)-dependent constants, the lower bound proved under the
stronger guarantee implies the stated result.

The realizable lower bound of Theorem~\ref{thm:main_lower_realizable} supplies the first term in
\(\Psi_{k,\rho}^{\agn}\). For the genuinely agnostic term, we first
establish a localized ordinary excess-error lower bound at a prescribed
probability mass scale \(m_0\), and then transfer it through the 
lower-increment bound for the event-inflation map in Theorem~\ref{thm:qp-qp'_bounds}. Setting
\[
m_0=p_c(\rho;k)
\]
will yield the \(d\)-dependent contribution, while a separate
two-point construction will yield the confidence contribution.

\begin{lemma}[Localized traditional VC lower bound]
\label{lem:localized_massart_cube}
There exist universal constants
\(c_{\mathrm{trad}}\in (0,1/8]\) and \(\delta_{\mathrm{cube}}>0\) such
that the following holds.  Let \(\hypoclass\) be a hypothesis class with
\(\VC(\hypoclass)\ge d\ge2\), let \(m_0\in(0,1]\), and let \(n\ge1\).
Then, there exists a family of distributions
\[
\{P_b:b\in\{-1,1\}^{d-1}\}
\]
such that, for every \(b\in\{-1,1\}^{d-1}\),
\[
p_b^\star
\coloneqq
\inf_{h\in\hypoclass}\err{h}{P_b}
\in
\left[\frac{m_0}{4},\frac{m_0}{2}\right],
\]
and, for every learning rule
\(r_n:(\X\times\Y)^n\to\mathcal G\),
\[
\sup_{b\in\{-1,1\}^{d-1}}
\Pr_{S_n\sim P_b^n}\!\left[
\err{r_n(S_n)}{P_b}
-
p_b^\star
\ge
c_{\mathrm{trad}}\cdot
\min\left\{
m_0,\sqrt{\frac{m_0d}{n}}
\right\}
\right]
\ge
\delta_{\mathrm{cube}} .
\]
\end{lemma}

When \(m_0\leq p_c(\rho;k)\), the next lemma transfers this localized
ordinary lower bound through the lower-increment estimate of
Theorem~\ref{thm:qp-qp'_bounds}.

\begin{lemma}[Localized robust lower bound]
\label{lem:localized_robust_lower_mass}
There exists a universal constant
\(\delta_{\mathrm{rob}}>0\) such that, for every \(k>1\), there
exists a constant \(c_{\mathrm{rob},k}>0\)
for which the following holds.
Let \(\rho\ge0\), let
\(\VC(\hypoclass)\ge d\ge2\), and let \(n\ge1\).  For every
\(m_0\in(0,p_c(\rho;k)]\) and every learning rule
\(r_n:(\X\times\Y)^n\to\mathcal G\), there exists a distribution \(P\)
such that
\[
\begin{aligned}
\Pr_{S_n\sim P^n}\!\Bigg[
&
\risk{k,\rho}{r_n(S_n)}{P}
-
\inf_{h\in\hypoclass}
\risk{k,\rho}{h}{P}
\\
&\qquad\qquad\qquad\ge
c_{\mathrm{rob},k}
\min\left\{
m_0,
\sqrt{\frac{m_0d}{n}}
\right\}
\left(
1+
\left(\frac{\rho}{m_0}\right)^{1/k}
\right)
\Bigg]
\ge
\delta_{\mathrm{rob}} .
\end{aligned}
\]
\end{lemma}

\medskip
We now prove the theorem.  Let \(c_{\mathrm{rob},k}\) and
\(\delta_{\mathrm{rob}}\) be the constants declared in
Lemma~\ref{lem:localized_robust_lower_mass}.  By decreasing
\(c_{\mathrm{rob},k}\), if necessary, we also assume that
\(c_{\mathrm{rob},k}\le1\) and that the lower-increment bound in
Theorem~\ref{thm:qp-qp'_bounds} holds with the same constant
\(c_{\mathrm{rob},k}\).  
Using Lemma~\ref{lem:agnostic_endpoint_comparison}, choose
\(a_k\in(0,1]\), depending only on \(k\), sufficiently small that,
for every \(\rho\ge0\),
\[
p_c(\rho;k)
\left[
1+
\left(
\frac{\rho}{p_c(\rho;k)}
\right)^{1/k}
\right]
\ge
a_k
\]
and
\[
p_c(\rho;k)
\left[
1+
\left(
\frac{\rho}{p_c(\rho;k)}
\right)^{1/k}
\right]^2
\ge
a_k
\left(
1+\rho^{1/(k-1)}
\right).
\]
Let
\(\delta_{\mathrm{real}}\) and \(\eps_{\mathrm{real}}\) be the constants
appearing in Theorem~\ref{thm:main_lower_realizable}.  The constant
\(\tilde c_k\) in the statement of the theorem will be chosen at the end
of the proof, and may be decreased finitely many times.

Set
\[
\delta_0
\coloneqq
\frac12
\min\left\{
\delta_{\mathrm{rob}},
\delta_{\mathrm{real}},
\frac18
\right\},
\]
which is universal, and choose \(\tilde\eps_k>0\) so that
\[
\tilde\eps_k
\le
\min\left\{
\eps_{\mathrm{real}},
\frac{a_kc_{\mathrm{rob},k}}{2}
\right\}.
\]
Fix
\[
\rho\ge0,
\qquad
\delta\in(0,\delta_0],
\qquad
\eps\in(0,\tilde\eps_k],
\]
and suppose that
\(r_n:(\X\times\Y)^n\to\mathcal G\) satisfies the agnostic guarantee in
the statement of the theorem; namely, for every distribution \(P\) over
\(\X\times\Y\),
\begin{equation}
\label{eq:agnostic_lower_assumed_guarantee}
\Pr_{S_n\sim P^n}\!\left[
\risk{k,\rho}{r_n(S_n)}{P}
-
\inf_{h\in\hypoclass}\risk{k,\rho}{h}{P}
\ge \eps
\right]
\le \delta .
\end{equation}

Since~\eqref{eq:agnostic_lower_assumed_guarantee} holds for every
distribution, it holds in particular for every realizable
distribution. For such distributions the optimal robust risk is zero,
so Theorem~\ref{thm:main_lower_realizable} gives, after decreasing
\(\tilde c_k\) if necessary,
\begin{equation}
\label{eq:agnostic_realizable_contribution}
n
\ge
\tilde c_k\,
\Psi_{k,\rho}^{\real}(\eps)
\bigl(d+\log(1/\delta)\bigr).
\end{equation}

It remains to prove the genuinely agnostic contribution, namely
\[
n
\ge
\tilde c_k\,
\frac{1+\rho^{1/(k-1)}}{\eps^2}
\bigl(d+\log(1/\delta)\bigr),
\]
again after possibly decreasing \(\tilde c_k\).

Suppose first that \(d\ge2\). Apply
Lemma~\ref{lem:localized_robust_lower_mass} with
\[
m_0=p_c(\rho;k).
\]
Since \(\delta<\delta_{\mathrm{rob}}\), the assumed guarantee
\eqref{eq:agnostic_lower_assumed_guarantee} forces
\begin{equation}
\label{eq:agnostic_endpoint_constraint}
c_{\mathrm{rob},k}
\min\left\{
p_c(\rho;k),
\sqrt{p_c(\rho;k)\frac dn}
\right\}
\left[
1+
\left(\frac{\rho}{p_c(\rho;k)}\right)^{1/k}
\right]
\leq
\eps;
\end{equation}
otherwise, the lower event supplied by the lemma would contradict
\eqref{eq:agnostic_lower_assumed_guarantee}.

The minimum in \eqref{eq:agnostic_endpoint_constraint} cannot equal
\(p_c(\rho;k)\), since then its left-hand side would be at least
\[
c_{\mathrm{rob},k}a_k
\geq
2\eps,
\]
by the choice of \(\tilde\eps_k\). Hence
\[
\frac dn
\leq
p_c(\rho;k),
\]
and the minimum equals
\(\sqrt{p_c(\rho;k)d/n}\). Squaring
\eqref{eq:agnostic_endpoint_constraint} and using the second defining
property of \(a_k\) gives, after decreasing \(\tilde c_k\) if
necessary,
\begin{equation}
\label{eq:agnostic_endpoint_d_lower}
n
\geq
\tilde c_k\,
d\,
\frac{1+\rho^{1/(k-1)}}{\eps^2}.
\end{equation}

We next prove the corresponding confidence-dependent bound, which is
valid for every \(d\ge1\).
By the standing assumption \(|\hypoclass|\ge3\), choose distinct
hypotheses \(h_+,h_-\in\hypoclass\) that are not pointwise
complements; such a pair exists because a binary hypothesis has at
most one pointwise complement.  Choose \(x_0,x_1\in\X\) at which they
respectively disagree and agree, and index the pair so that
\[
h_\sigma(x_0)=\sigma,
\qquad
h_+(x_1)=h_-(x_1)=:y_1,
\qquad
\sigma\in\{-1,1\}.
\]
For \(\sigma\in\{-1,1\}\), define a distribution \(P_\sigma\) on
\(\X\times\Y\) by
\[
P_\sigma(X=x_0,Y=\sigma)
=
p_c(\rho;k)\frac{1+\gamma}{2},
\qquad
P_\sigma(X=x_0,Y=-\sigma)
=
p_c(\rho;k)\frac{1-\gamma}{2},
\]
and
\[
P_\sigma(X=x_1,Y=y_1)
=
1-p_c(\rho;k),
\]
where
\[
\gamma
\coloneqq
\frac{\eps}{
c_{\mathrm{rob},k}
p_c(\rho;k)
\left[
1+\left(\frac{\rho}{p_c(\rho;k)}\right)^{1/k}
\right]}.
\]

By the defining property of \(a_k\) and the choice of
\(\tilde\eps_k\),
\[
\gamma
\le
\frac{\eps}{c_{\mathrm{rob},k}a_k}
\le
\frac12 .
\]

For \(P_\sigma\), \(h_\sigma\) is Bayes-optimal, and
\[
\err{h_\sigma}{P_\sigma}
=
p_c(\rho;k)\frac{1-\gamma}{2},
\qquad
\err{g}{P_\sigma}
\ge
p_c(\rho;k)\frac{1+\gamma}{2}
\quad
\text{for every }g\in\mathcal G
\text{ with }g(x_0)\neq\sigma.
\]
Both error levels lie below \(p_c(\rho;k)\). Hence the
one-dimensional event-inflation identity and monotonicity imply that
every \(g\in\mathcal G\) with \(g(x_0)\neq\sigma\) has robust excess
risk at least
\[
\q{k,\rho}{
p_c(\rho;k)\tfrac{1+\gamma}{2}
}
-
\q{k,\rho}{
p_c(\rho;k)\tfrac{1-\gamma}{2}
}.
\]
For \(\rho>0\), the lower-increment bound in
Theorem~\ref{thm:qp-qp'_bounds} bounds this difference below by
\[
c_{\mathrm{rob},k}
p_c(\rho;k)\gamma
\left[
1+
\left(\frac{\rho}{p_c(\rho;k)}\right)^{1/k}
\right]
=
\eps.
\]
For \(\rho=0\), the same conclusion follows from
\(\q{k,0}{p}=p\) and \(c_{\mathrm{rob},k}\leq1\).
Therefore,
\[
P_\sigma^n\!\left(
r_n(S_n)(x_0)\neq\sigma
\right)
\leq
\delta,
\qquad
\sigma\in\{-1,1\}.
\]

Thus \(S_n\mapsto r_n(S_n)(x_0)\) is a test between
\(P_+^n\) and \(P_-^n\) with both errors at most \(\delta\).
The Bretagnolle--Huber inequality
\citep[Lemma~2.6]{tsybakov2009introduction} gives
\[
2\delta
\geq
\frac12
\exp\!\left(
-n\DKL(P_+\Vert P_-)
\right),
\]
and therefore, since \(\delta\leq1/16\),
\[
n\DKL(P_+\Vert P_-)
\geq
\frac12\log\frac1\delta.
\]
Moreover,
\[
\DKL(P_+\Vert P_-)
=
p_c(\rho;k)\gamma
\log\frac{1+\gamma}{1-\gamma}
\leq
4p_c(\rho;k)\gamma^2,
\]
where the inequality uses \(\gamma\leq1/2\). Consequently,
\[
n
\geq
\frac{1}{8p_c(\rho;k)\gamma^2}
\log\frac1\delta
=
\frac{c_{\mathrm{rob},k}^2}{8}
\frac{
p_c(\rho;k)
\left[
1+
\left(\frac{\rho}{p_c(\rho;k)}\right)^{1/k}
\right]^2
}{\eps^2}
\log\frac1\delta.
\]
Using the second defining property of \(a_k\), and decreasing
\(\tilde c_k\) if necessary, yields
\begin{equation}
\label{eq:agnostic_endpoint_delta_lower}
n
\geq
\tilde c_k
\frac{1+\rho^{1/(k-1)}}{\eps^2}
\log\frac1\delta.
\end{equation}

If \(d\ge2\), combining
\eqref{eq:agnostic_endpoint_d_lower} and
\eqref{eq:agnostic_endpoint_delta_lower} gives the desired bound.
If \(d=1\), then \(\delta\le\delta_0\le1/16\) implies
\[
d+\log(1/\delta)
=
1+\log(1/\delta)
\le
\left(
1+\frac{1}{\log 16}
\right)
\log(1/\delta),
\]
so \eqref{eq:agnostic_endpoint_delta_lower} gives the same conclusion.
Thus, after decreasing \(\tilde c_k\) if necessary, in either case
\begin{equation}
\label{eq:agnostic_endpoint_full_lower}
n
\ge
\tilde c_k\,
\frac{1+\rho^{1/(k-1)}}{\eps^2}
\bigl(d+\log(1/\delta)\bigr).
\end{equation}
Together with the realizable contribution
\eqref{eq:agnostic_realizable_contribution}, this gives
\[
n
\ge
\tilde c_k
\max\left\{
\Psi_{k,\rho}^{\real}(\eps),
\frac{1+\rho^{1/(k-1)}}{\eps^2}
\right\}
\bigl(d+\log(1/\delta)\bigr).
\]
By the definition of \(\Psi_{k,\rho}^{\agn}\), this is
\[
n
\ge
\tilde c_k\,
\Psi_{k,\rho}^{\agn}(\eps)
\bigl(d+\log(1/\delta)\bigr),
\]
which proves the theorem.
\hfill\BlackBox

\subsubsection{
Proof of Lemma~\ref{lem:localized_massart_cube}: Localized Traditional VC Lower Bound}
\label{sec:lem:localized_massart_cube}

Choose a shattered set \(\{x_1,\dots,x_d\}\subset\X\), and set
\(
\pi_0\coloneqq\frac{m_0}{d-1}.
\)
For each \(b=(b_1,\dots,b_{d-1})\in\{-1,1\}^{d-1}\), define
\(P_b\) as follows.  Let
\[
P_b(X=x_i)=\pi_0,\qquad 1\le i\le d-1,
\]
and
\(
P_b(X=x_d)=1-m_0.
\)
Conditionally on \(X=x_i\), for \(1\le i\le d-1\), set
\[
P_b(Y=b_i\mid X=x_i)=\frac{1+\gamma}{2},
\qquad
P_b(Y=-b_i\mid X=x_i)=\frac{1-\gamma}{2},
\]
where \(\gamma\in(0,1/2]\) will be chosen below.  At the anchor point,
set
\(
P_b(Y=-1\mid X=x_d)=1.
\)

Since the \(d\) points are shattered, there exists
\(h_b^\star\in\hypoclass\) such that
\[
h_b^\star(x_i)=b_i,\qquad 1\le i\le d-1,
\qquad\text{and}\qquad
h_b^\star(x_d)=-1.
\]
This classifier is Bayes optimal under \(P_b\).  Hence
\[
p_b^\star
=
\inf_{h\in\hypoclass}\err{h}{P_b}
=
\err{h_b^\star}{P_b}
=
m_0\frac{1-\gamma}{2}.
\]
Since \(\gamma\le1/2\),
\(
p_b^\star\in\left[{m_0}/{4},{m_0}/{2}\right].
\)
For every \(g\in\mathcal G\),
\[
\err{g}{P_b}-p_b^\star
=
\pi_0\gamma
\sum_{i=1}^{d-1}\ind{g(x_i)\ne b_i}
+
(1-m_0)\ind{g(x_d)\ne -1}.
\]
Indeed, an incorrect prediction at an active point increases the
conditional error by \(\gamma\), while the anchor contribution is
nonnegative. For a learning rule \(r_n\), define
\[
\hat b_n(S_n)
\coloneqq
\bigl(r_n(S_n)(x_1),\dots,r_n(S_n)(x_{d-1})\bigr).
\]
Dropping the nonnegative anchor contribution gives
\begin{equation}
\label{eq:localized_massart_excess_hamming}
\err{r_n(S_n)}{P_b}
-
p_b^\star
\ge
\pi_0\gamma\,\hamming(\hat b_n(S_n),b).
\end{equation}

We next lower-bound the Hamming error.  
If \(b^{(j)}\) is obtained from \(b\) by flipping coordinate \(j\),
then \(P_b\) and \(P_{b^{(j)}}\) differ only at \(x_j\), and for
\(\gamma\leq1/2\),
\[
\DKL(P_b\Vert P_{b^{(j)}})
=
\pi_0\gamma
\log\frac{1+\gamma}{1-\gamma}
\leq
4\pi_0\gamma^2.
\]
Choose
\[
\gamma
\coloneqq
\min\left\{
\frac12,
\frac18\sqrt{\frac{d-1}{nm_0}}
\right\}.
\]
Then
\[
n\DKL(P_b\Vert P_{b^{(j)}})\leq\frac14,
\]
and hence, by Pinsker's inequality,
\[
\left\|
P_b^n-P_{b^{(j)}}^n
\right\|_{\mathrm{TV}}
\leq\frac12.
\]
Assouad's lemma \citep{yu1997assouad} therefore gives
\[
\sup_b
\EE_{P_b^n}
\left[
\hamming(\hat b_n,b)
\right]
\geq
\frac{d-1}{4}.
\]
Since
\[
0
\leq
\hamming(\hat b_n(S_n),b)
\leq
d-1,
\]
the preceding expectation bound implies
\[
\sup_{b\in\{-1,1\}^{d-1}}
\Pr_{S_n\sim P_b^n}\!\left[
\hamming(\hat b_n(S_n),b)
\geq
\frac{d-1}{8}
\right]
\geq
\frac17.
\]
Together with
\eqref{eq:localized_massart_excess_hamming} and
\(\pi_0(d-1)=m_0\), this yields
\[
\sup_{b\in\{-1,1\}^{d-1}}
\Pr_{S_n\sim P_b^n}\!\left[
\err{r_n(S_n)}{P_b}
-
p_b^\star
\geq
\frac{m_0\gamma}{8}
\right]
\geq
\frac17.
\]
Moreover, since \(d\geq2\),
\begin{align*}
m_0\gamma
&=
\min\left\{
\frac{m_0}{2},
\frac18
\sqrt{\frac{m_0(d-1)}{n}}
\right\}
\geq
\frac{1}{8\sqrt2}
\min\left\{
m_0,
\sqrt{\frac{m_0d}{n}}
\right\}.
\end{align*}
Thus the lemma holds with, for example,
\[
c_{\mathrm{trad}}
=
\frac{1}{64\sqrt2},
\qquad
\delta_{\mathrm{cube}}
=
\frac17.
\]
\hfill\BlackBox

\subsubsection{
Proof of Lemma~\ref{lem:localized_robust_lower_mass}: Localized Robust Lower Bound}
\label{sec:lem:localized_robust_lower_mass}

Fix
\[
\rho\geq0,
\qquad
m_0\in(0,p_c(\rho;k)],
\]
and a learning rule
\(r_n:(\X\times\Y)^n\to\mathcal G\).
Apply Lemma~\ref{lem:localized_massart_cube} with this value of
\(m_0\), and set
\[
\xi_n(m_0)
\coloneqq
c_{\mathrm{trad}}
\min\left\{
m_0,
\sqrt{\frac{m_0d}{n}}
\right\}.
\]
For some \(b\in\{-1,1\}^{d-1}\), writing \(P=P_b\),
\begin{equation}
\label{eq:localized_ordinary_event}
\Pr_{S_n\sim P^n}\!\left[
\err{r_n(S_n)}{P}
-
p_b^\star
\geq
\xi_n(m_0)
\right]
\geq
\delta_{\mathrm{cube}} .
\end{equation}
Moreover,
\[
p_b^\star
\in
\left[
\frac{m_0}{4},
\frac{m_0}{2}
\right],
\qquad
\xi_n(m_0)
\leq
\frac{m_0}{8},
\]
and hence
\[
0
<
p_b^\star
<
p_b^\star+\xi_n(m_0)
<
m_0
\leq
p_c(\rho;k).
\]

By the one-dimensional event-inflation identity and monotonicity,
\[
\inf_{h\in\hypoclass}
\risk{k,\rho}{h}{P}
=
\q{k,\rho}{p_b^\star},
\]
and on the event in
\eqref{eq:localized_ordinary_event},
\[
\begin{aligned}
\risk{k,\rho}{r_n(S_n)}{P}
-
\inf_{h\in\hypoclass}
\risk{k,\rho}{h}{P}
\geq
\q{k,\rho}{
p_b^\star+\xi_n(m_0)
}
-
\q{k,\rho}{p_b^\star}.
\end{aligned}
\]
If \(\rho=0\), the right-hand side equals
\(\xi_n(m_0)\). If \(\rho>0\), the lower-increment bound in
Theorem~\ref{thm:qp-qp'_bounds} gives
\begin{align*}
\q{k,\rho}{
p_b^\star+\xi_n(m_0)
}
-
\q{k,\rho}{p_b^\star}
&\geq
\tilde c_k
\xi_n(m_0)
\left[
1+
\left(
\frac{\rho}{
p_b^\star+\xi_n(m_0)
}
\right)^{1/k}
\right]
\\
&\geq
\tilde c_k
\xi_n(m_0)
\left[
1+
\left(
\frac{\rho}{m_0}
\right)^{1/k}
\right].
\end{align*}
Substituting the definition of \(\xi_n(m_0)\) proves the lemma,
after decreasing \(c_{\mathrm{rob},k}\) if necessary, with
\[
\delta_{\mathrm{rob}}
=
\delta_{\mathrm{cube}}.
\]
\hfill\BlackBox

\subsubsection{
Proof of Lemma~\ref{lem:agnostic_endpoint_comparison}: Endpoint Comparison}
\label{sec:lem:agnostic_endpoint_comparison}

Set
\[
t
\coloneqq
\rho^{1/(k-1)}.
\]
Since
\[
1+k(k-1)t^{k-1}
\asymp_k
(1+t)^{k-1},
\]
we have
\[
p_c(\rho;k)
=
\bigl(1+k(k-1)\rho\bigr)^{-1/(k-1)}
\asymp_k
\frac1{1+t}.
\]
Consequently,
\begin{align*}
&
p_c(\rho;k)
+
\rho^{1/k}
p_c(\rho;k)^{1-1/k}
\asymp_k
\frac1{1+t}
+
\left(\frac{t}{1+t}\right)^{(k-1)/k}
\asymp_k
1.
\end{align*}
Both terms are at most \(1\). Moreover, the first is at least
\(1/2\) when \(t\leq1\), whereas the second is at least
\(2^{-(k-1)/k}\) when \(t\geq1\).

Similarly,
\begin{align*}
p_c(\rho;k)^{1/2}
+
\rho^{1/k}
p_c(\rho;k)^{1/2-1/k}
&\asymp_k
(1+t)^{1/2}
\left[
\frac1{1+t}
+
\left(\frac{t}{1+t}\right)^{(k-1)/k}
\right]
\\
&\asymp_k
(1+t)^{1/2}
=
\left(
1+\rho^{1/(k-1)}
\right)^{1/2}.
\end{align*}
This proves both comparisons.
\hfill\BlackBox

\subsection{Self-Consistency Lemma}
\label{sec:self-cons}

\begin{lemma}[Self-consistency of the sample-size condition]
\label{lem:self-consistency}
Let \(\Gamma_{d,n}(\delta)\) be as in
\eqref{eq:Gamma_def}, and let \(\Psi\geq1\). There exists a universal
constant \(C>0\) 
such that, for all integers \(n\geq d\geq1\) and every
\(\delta\in(0,1)\),
\[
n
\geq
C\Psi
\left(
d\log(\e\Psi)+\log\delta^{-1}
\right)
\quad\Longrightarrow\quad
n
\geq
\Psi\Gamma_{d,n}(\delta).
\]
\end{lemma}

\begin{proof}
Set
\[
x\coloneqq\frac{n}{d}
\qquad\text{and}\qquad
L
\coloneqq
\Psi\log(\e\Psi)
+
\frac{\Psi}{d}\log\left(\frac{8}{\delta}\right).
\]
We first note that the premise of the lemma, after increasing its
universal constant if necessary, implies
\[
x\geq C_0L
\]
for any prescribed universal constant \(C_0\). Indeed, if
\[
A
\coloneqq
\Psi\log(\e\Psi)
+
\frac{\Psi}{d}\log\delta^{-1},
\]
then \(A\geq\Psi\), and hence
\[
L
=
A+\frac{\Psi}{d}\log 8
\leq
A+\Psi\log 8
\leq
(1+\log 8)A.
\]

It therefore suffices to show that, for a sufficiently large
universal constant \(C_0\),
\[
x\geq C_0L
\quad\Longrightarrow\quad
x
\geq
\Psi\log(2\e x)
+
\frac{\Psi}{d}\log\left(\frac{8}{\delta}\right).
\]
The latter inequality is equivalent to
\(n\geq\Psi\Gamma_{d,n}(\delta)\).

Define
\[
F(t)
\coloneqq
t
-
\Psi\log(2\e t)
-
\frac{\Psi}{d}\log\left(\frac{8}{\delta}\right).
\]
Then
\[
F'(t)=1-\frac{\Psi}{t},
\]
so \(F\) is increasing on \([2\Psi,\infty)\). Moreover,
\(L\geq\Psi\), and therefore \(C_0L\geq2\Psi\) whenever
\(C_0\geq2\).

We next bound \(F(C_0L)\). Since \(L\geq\Psi\),
\begin{align*}
\Psi\log L
&=
\Psi\log\Psi
+
\Psi\log\left(\frac{L}{\Psi}\right)
\\
&\leq
\Psi\log(\e\Psi)+L
\\
&\leq
2L,
\end{align*}
where we used
\(\log u\leq u\) for \(u\geq1\). Also,
\[
\frac{\Psi}{d}\log\left(\frac{8}{\delta}\right)
\leq L
\qquad\text{and}\qquad
\Psi\log(2\e C_0)
\leq
\log(2\e C_0)L.
\]
Consequently,
\begin{align*}
F(C_0L)
&=
C_0L
-
\Psi\log(2\e C_0)
-
\Psi\log L
-
\frac{\Psi}{d}\log\left(\frac{8}{\delta}\right)
\\
&\geq
\left(
C_0-3-\log(2\e C_0)
\right)L.
\end{align*}
Choosing \(C_0\) sufficiently large makes the last expression
nonnegative. Since \(F\) is increasing on
\([2\Psi,\infty)\), every \(x\geq C_0L\) therefore satisfies
\(F(x)\geq0\). This proves
\[
n\geq\Psi\Gamma_{d,n}(\delta).
\]
\end{proof}

\bibliography{ref_jmlr_submission}
\end{document}